\pdfoutput=1

\documentclass{article}

\usepackage{times}
\usepackage{graphicx} 
\usepackage{subfigure}
\usepackage{float}

\usepackage{algorithm}
\usepackage{algorithmic}

\usepackage[accepted]{icml2017}

\usepackage[utf8]{inputenc} 
\usepackage{url}            
\usepackage{booktabs}       
\usepackage{amsfonts}       
\usepackage{nicefrac}       

\usepackage[T1]{fontenc}

\usepackage{amsthm,amsmath,verbatim}
\usepackage{amsfonts}

\usepackage{graphicx}
\usepackage{adjustbox}

\usepackage{color}

\theoremstyle{plain}
\newtheorem{theorem}{Theorem}
\newtheorem{corollary}{Corollary}
\newtheorem{lemma}{Lemma}

\theoremstyle{definition}

\newtheorem{obs}{Observation}

\newcommand{\norm}[1]{\left|\left| #1 \right|\right|}

\usepackage{xcolor}
\definecolor{darkgreen}{rgb}{0,0.6,0}
\definecolor{darkred}{rgb}{0.6,0,0}
\definecolor{darkblue}{rgb}{0.0,0,0.6}

\newcommand{\node}[2]{v^{(#1)}_{#2}}

\newcommand{\hidden}[2]{h^{(#1)}_{#2}}

\newcommand{\mbold}[1]{\mathbb{#1}}

\newcommand{\perperp}[1]{{}^{\perp}#1_{\perp}}

\newcommand{\parpar}[1]{{}^{\parallel}#1_{\parallel}}
\newcommand{\parperp}[1]{{}^{\parallel}#1_{\perp}}
\newcommand{\perpar}[1]{{}^{\perp}#1_{\parallel}}

\newcommand{\perpart}[1]{#1_{\perp}}
\newcommand{\parpart}[1]{#1_{\parallel}}

\makeatletter
\@fpsep\textheight
\makeatother

\icmltitlerunning{On the Expressive Power of Deep Neural Networks}

\begin{document}

\twocolumn[
\icmltitle{On the Expressive Power of Deep Neural Networks}

\begin{icmlauthorlist}
\icmlauthor{Maithra Raghu}{co,goo}
\icmlauthor{Ben Poole}{st}
\icmlauthor{Jon Kleinberg}{co}
\icmlauthor{Surya Ganguli}{st}
\icmlauthor{Jascha Sohl Dickstein}{goo}
\end{icmlauthorlist}

\icmlaffiliation{co}{Cornell University}
\icmlaffiliation{goo}{Google Brain}
\icmlaffiliation{st}{Stanford University}

\icmlcorrespondingauthor{Maithra Raghu}{maithrar@gmail.com}
\vskip 0.3in

]

\printAffiliationsAndNotice{}

\begin{abstract}
    We propose a new approach to the problem of \textit{neural network expressivity}, which seeks to characterize how structural properties of a neural network family affect the functions it is able to compute. Our approach is based on an interrelated set of measures of expressivity, unified by the novel notion of \textit{trajectory length}, which measures how the output of a network changes as the input sweeps along a one-dimensional path. Our findings can be summarized as follows:
    \begin{enumerate}
      \item[(1)] \textit{The complexity of the computed function grows exponentially with depth.} We design measures of expressivity that capture the non-linearity of the computed function. Due to how the network transforms its input, these measures grow exponentially with depth.
      \item[(2)] \textit{All weights are not equal (initial layers matter more).} We find that trained networks are far more sensitive to their lower (initial) layer weights: they are much less robust to noise in these layer weights, and also perform better when these weights are optimized well.
      \item[(3)] \textit{Trajectory Regularization works like Batch Normalization.} We find that batch norm stabilizes the learnt representation, and based on this propose a new regularization scheme, \textit{trajectory regularization.}
    \end{enumerate}
\end{abstract}

\section{Introduction}

Deep neural networks have proved astoundingly effective at a wide range of empirical tasks, from image classification \citep{krizhevsky2012imagenet} to 
playing Go \citep{silver2016mastering}, and even modeling human learning \citep{piech2015deep}.

Despite these successes, understanding of how and why neural network architectures achieve their empirical successes is still lacking. This includes even the fundamental question of \textit{neural network expressivity}, how the architectural properties of a neural network (depth, width, layer type) affect the resulting functions it can compute, and its ensuing performance.

This is a foundational question, and there is a rich history of prior work
addressing expressivity in neural networks. However, it has been challenging to derive conclusions that
provide both theoretical generality with respect to choices of
architecture as well as meaningful insights into practical performance.

Indeed, the very first results on this question take a highly theoretical approach, from using functional analysis to show universal approximation results \citep{hornik1989multilayer,cybenko1989approximation}, to analysing expressivity via comparisons to boolean circuits \citep{maass1994comparison} and studying network VC dimension \citep{bartlett1998almost}. While these results provided theoretically general conclusions, the shallow networks they studied are very different from the deep models that have proven so successful in recent years.

In response, several recent papers have focused on understanding the benefits of depth for neural networks \citep{pascanu2013number,montufar2014number,eldan2015power,telgarsky2015representation,martens2013representational,bianchini2014complexity}. These results are compelling and take modern architectural changes into account, but they only show that a specific choice of weights for a deeper network results in inapproximability by a shallow (typically one or two hidden layers) network.

In particular, the goal of this new line of work has been to establish
lower bounds --- showing separations between shallow and deep networks ---
and as such they are based on hand-coded constructions of specific
network weights.  Even if the weight values used in these constructions
are robust to small perturbations (as in
\citep{pascanu2013number,montufar2014number}), the functions
that arise from these constructions tend toward extremal properties
by design, and there is no evidence
that a network trained on data ever resembles such a function.

This has meant that a set of fundamental questions about neural network expressivity has remained largely unanswered.
First, we lack a good understanding of the ``typical'' case
rather than the worst case in these bounds for deep networks,
and consequently have no way to evaluate whether the hand-coded
extremal constructions provide a reflection of the complexity
encountered in more standard settings.
Second, we lack an understanding of upper bounds to match the
lower bounds produced by this prior work; 
do the constructions used to date place us near the limit of
the expressive power of neural networks,
or are there still large gaps?
Finally, if we had an understanding of these two issues, we might begin
to draw connections between network expressivity and observed performance.

\paragraph*{\bf Our contributions: Measures of Expressivity and their Applications}  In this paper, we address this set of challenges by defining
and analyzing an interrelated set of {\em measures of expressivity}
for neural networks; our framework applies to 
a wide range of standard architectures, 
independent of specific weight choices.  
We begin our analysis at the start of training, after random initialization, and later derive insights connecting network expressivity and performance.

Our first measure of expressivity is based on the notion of
an {\em activation pattern}: in a network where the units compute functions
based on discrete thresholds, we can ask which units are above or below
their thresholds (i.e. which units are ``active'' and which are not).
For the range of standard architectures that we consider, 
the network is essentially computing a linear function once we fix
the activation pattern; thus, counting the number of possible 
activation patterns provides a concrete way of measuring the complexity
beyond linearity that the network provides.
We give an upper bound on the number of possible activation patterns,
over any setting of the weights. This bound is tight as it matches the hand-constructed lower bounds
of earlier work \citep{pascanu2013number,montufar2014number}.

Key to our analysis is the notion of a {\em transition}, in which 
changing an input $x$ to a nearby input $x + \delta$ changes the 
activation pattern.  
We study the behavior of transitions as we pass the input along
a one-dimensional parametrized trajectory $x(t)$.
Our central finding is that the {\em trajectory length} 
grows exponentially in the depth of the network.

Trajectory length serves as a unifying notion in our measures
of expressivity, and it leads to insights into the behavior of 
trained networks.
Specifically, we find that the exponential growth in trajectory
length as a function of depth implies that small adjustments 
in parameters lower in the network induce larger changes
than comparable adjustments higher in the network.
We demonstrate this phenomenon through
experiments on MNIST and CIFAR-10, where the network displays much
less robustness to noise in the lower layers, and better
performance when they are trained well. We also explore the effects of regularization methods on trajectory length as the network trains and propose a less computationally intensive method of regularization, \textit{trajectory regularization}, that offers the same performance as batch normalization. 

The contributions of this paper are thus:
\begin{enumerate}
    \item[(1)] \textit{Measures of expressivity}: We propose easily computable measures of neural network expressivity that capture the expressive power inherent in different neural network architectures, independent of specific weight settings.
    \item[(2)] \textit{Exponential trajectories:} We find an exponential depth dependence displayed by these measures, through a unifying analysis in which we study
how the network transforms its input by measuring \textit{trajectory length}
    \item[(3)] \textit{All weights are not equal (the lower layers
matter more)}: We show how these results on trajectory length
suggest that optimizing weights in lower layers of the network is 
particularly important.
\item[(4)] \textit{Trajectory Regularization} Based on understanding the effect of batch norm on trajectory length, we propose a new method of regularization, trajectory regularization, that offers the same advantages as batch norm, and is computationally more efficient.
\end{enumerate}

In prior work \cite{deep_chaos}, we 
studied the propagation of
\textit{Riemannian curvature} through random networks by developing a
mean field theory approach. Here, we take an approach grounded in computational
geometry, presenting measures with a combinatorial flavor and
explore the consequences during and after training.

\section{Measures of Expressivity}
Given a neural network of a certain architecture $A$ (some depth, width, layer types), we have an associated function, $F_A(x; W)$, where $x$ is an input and $W$ represents all the parameters of the network. Our goal is to understand how the behavior of $F_A(x; W)$ changes as $A$ changes, for values of $W$ that we might encounter during training, and across inputs $x$.

The first major difficulty comes from the high dimensionality of the input. Precisely quantifying the properties of $F_A(x;W)$ over the entire input space is intractable. As a tractable alternative, we study simple one dimensional \textit{trajectories} through input space.
More formally:

\textbf{Definition:} Given two points, $x_0, x_1 \in \mbold{R}^m$, we say $x(t)$ is a trajectory (between $x_0$ and $x_1$) if $x(t)$ is a curve parametrized by a scalar $t \in [0, 1]$, with $x(0) = x_0$ and $x(1) = x_1$. 

Simple examples of a trajectory would be a line ($x(t) = tx_1 + (1 - t)x_0$) or a circular arc ($x(t) = \cos(\pi t/2)x_0 + \sin(\pi t/2)x_1$), but in general $x(t)$ may be  more complicated, and potentially not expressible in closed form.

Armed with this notion of trajectories, we can begin to define measures of expressivity of a network $F_A(x; W)$ over trajectories $x(t)$. 

\subsection{Neuron Transitions and Activation Patterns}\label{subsec_hyperplanes}
In \citep{montufar2014number} the notion of a ``linear region'' is introduced. Given a neural network with piecewise linear activations (such as ReLU or hard tanh), the function it computes is also piecewise linear, a consequence of the fact that composing piecewise linear functions results in a piecewise linear function. So one way to measure the ``expressive power'' of different architectures $A$ is to count the number of linear pieces (regions), which determines how nonlinear the function is.

In fact, a change in linear region is caused by a \textit{neuron transition} in the output layer. More precisely:

\textbf{Definition} For fixed $W$, we say a neuron with piecewise linear region \textit{transitions} between inputs $x, x+ \delta$ if its activation function switches linear region between $x$ and $x + \delta$.

So a ReLU transition would be given by a neuron switching from off to on (or vice versa) and for hard tanh by switching between saturation at $-1$ to its linear middle region to saturation at $1$. For any generic trajectory $x(t)$, we can thus define $\mathcal{T}(F_A(x(t); W))$ to be the number of transitions undergone by output neurons (i.e. the number of linear regions) as we sweep the input $x(t)$. Instead of just concentrating on the output neurons however, we can look at this pattern over the \textit{entire} network. We call this an \textit{activation patten}:

\textbf{Definition} We can define $\mathcal{AP}(F_A(x; W))$ to be the \textit{activation pattern} -- a string of form $\{0, 1\}^{\text{num neurons}}$ (for ReLUs) and $\{-1, 0, 1\}^{\text{num neurons}}$ (for hard tanh) of the network encoding the linear region of the activation function of \textit{every} neuron, for an input $x$ and weights $W$. 

Overloading notation slightly, we can also define (similarly to transitions) $\mathcal{A}(F_A(x(t); W))$ as the number of distinct activation patterns as we sweep $x$ along $x(t)$. As each distinct activation pattern corresponds to a different linear function of the input, this combinatorial measure captures how much more expressive $A$ is over a simple linear mapping.

Returning to Montufar et al, they provide a construction i.e. a \textit{specific} set of weights $W_0$, that results in an exponential increase of linear regions with the depth of the architectures. They also appeal to Zaslavsky's theorem \citep{stanleyhyperplanes} from the theory of hyperplane arrangements to show that a shallow network, i.e. \textit{one} hidden layer, with the same number of parameters as a deep network, has a much smaller number of linear regions than the number achieved by their choice of weights $W_0$ for the deep network.

More formally, letting $A_1$ be a fully connected network with one hidden layer, and $A_l$ a fully connected network with the same number of parameters, but $l$ hidden layers, they show
\[ \forall W \mathcal{T}(F_{A_1}([0, 1]; W)) < \mathcal{T}(F_{A_1}([0, 1]; W_0) \tag{*} \]

We derive a much more general result by considering the `global' activation patterns over the \textit{entire} input space, and prove that for any fully connected network, with \textit{any} number of hidden layers, we can upper bound the number of linear regions it can achieve, over \textit{all} possible weight settings $W$. This upper bound is asymptotically \textit{tight}, matched by the construction given in \citep{montufar2014number}. Our result can be written formally as:

\begin{theorem} \label{thm act count}
\emph{(Tight) Upper Bound for Number of Activation Patterns}
Let $A_{(n,k)}$ denote a fully connected network with $n$ hidden layers of width $k$, and inputs in $\mbold{R}^m$. Then the number of activation patterns $\mathcal{A}(F_{A_{n,k}}(\mbold{R}^m; W)$ is upper bounded by $O(k^{mn})$ for ReLU activations, and $O((2k)^{mn})$ for hard tanh. 
\end{theorem}

From this we can derive a \textit{chain} of inequalities. Firstly, from the theorem above we find an upper bound of $\mathcal{A}(F_{A_{n,k}}(\mbold{R}^m; W))$ over all $W$, i.e. 
\[ \forall \hspace{0.5mm} W \hspace{1mm} \mathcal{A}(F_{A_{(n,k)}})(\mbold{R}^m; W) \leq U(n,k,m) 
.\]

Next, suppose we have $N$ neurons in total. Then we want to compare (for wlog ReLUs), quantities like $U(n', N/n', m)$ for different $n'$. 

But $U(n', N/n', m) = O((N/n')^{mn'})$, and so, noting that the maxima of $\left(\frac{a}{x}\right)^{mx}$ (for $a > e$) is $x = a/e$, we get, (for $n, k > e$), in comparison to (*), 

\[ U(1,N,m) < U(2, \frac{N}{2}, m) < \cdots \hspace{5mm}  \]
\[ \hspace{5mm} \cdots < U(n-1, \frac{N}{n-1}, m) < U(n, k, m) \]

We prove this
via 
an inductive proof on regions in a hyperplane arrangement. The proof can be found in the Appendix. 
As noted in the introduction, this result differs from earlier lower-bound constructions in that it is an upper bound that applies to {\em all} possible sets of weights.
Via our analysis, we also prove

\begin{figure}
\centering
\adjincludegraphics[width=0.9\linewidth]{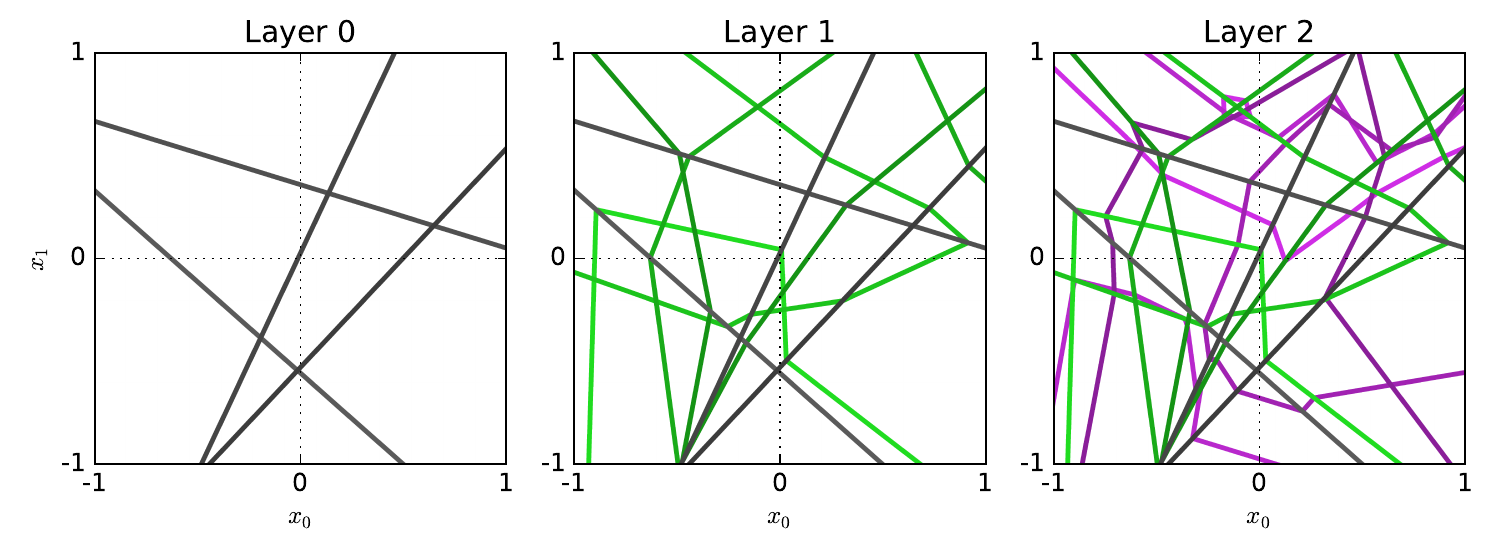}
\caption{\small Deep networks with piecewise linear activations subdivide input space into convex polytopes. 
We take a three hidden layer ReLU network, with input $x \in \mbold{R}^2$, and four units in each layer. The left pane shows activations for the first layer only. As the input is in $\mbold{R}^2$, neurons in the first hidden layer have an associated line in $\mbold{R}^2$, depicting their activation boundary. The left pane thus has four such lines. For the second hidden layer each neuron again has a line in input space corresponding to on/off, but this line is \textit{different} for each region described by the first layer activation pattern. So in the centre pane, which shows activation boundary lines corresponding to second hidden layer neurons in green (and first hidden layer in black), we can see the green lines `bend' at the boundaries. (The reason for this bending becomes apparent through the proof of Theorem \ref{thm_carve_inp_space}.) Finally, the right pane adds the on/off boundaries for neurons in the third hidden layer, in purple. These lines can bend at both black and green boundaries, as the image shows. This final set of convex polytopes corresponds to all activation patterns for this network (with its current set of weights) over the unit square, with each polytope representing a different linear function.
}
\label{fig convex polytope}
\end{figure}

\begin{theorem}\label{thm_carve_inp_space}
\emph{Regions in Input Space}
Given the corresponding function of a neural network $F_{A}(\mbold{R}^m; W)$ with  ReLU or hard tanh activations, the input space is partitioned into convex polytopes, with $F_{A}(\mbold{R}^m; W)$ corresponding to a different \textit{linear} function on each region. 
\end{theorem}

This result is of independent interest for optimization -- a linear function over a convex polytope results in a well behaved loss function and an easy optimization problem. Understanding the density of these regions during the training process would likely shed light on properties of the loss surface, and improved optimization methods. A picture of a network's regions 
is shown in Figure \ref{fig convex polytope}.

\subsubsection{Empirically Counting Transitions}

\begin{figure}
\centering
\begin{tabular}{cc}
\hspace*{-1cm}
\adjincludegraphics[width=0.6\linewidth]{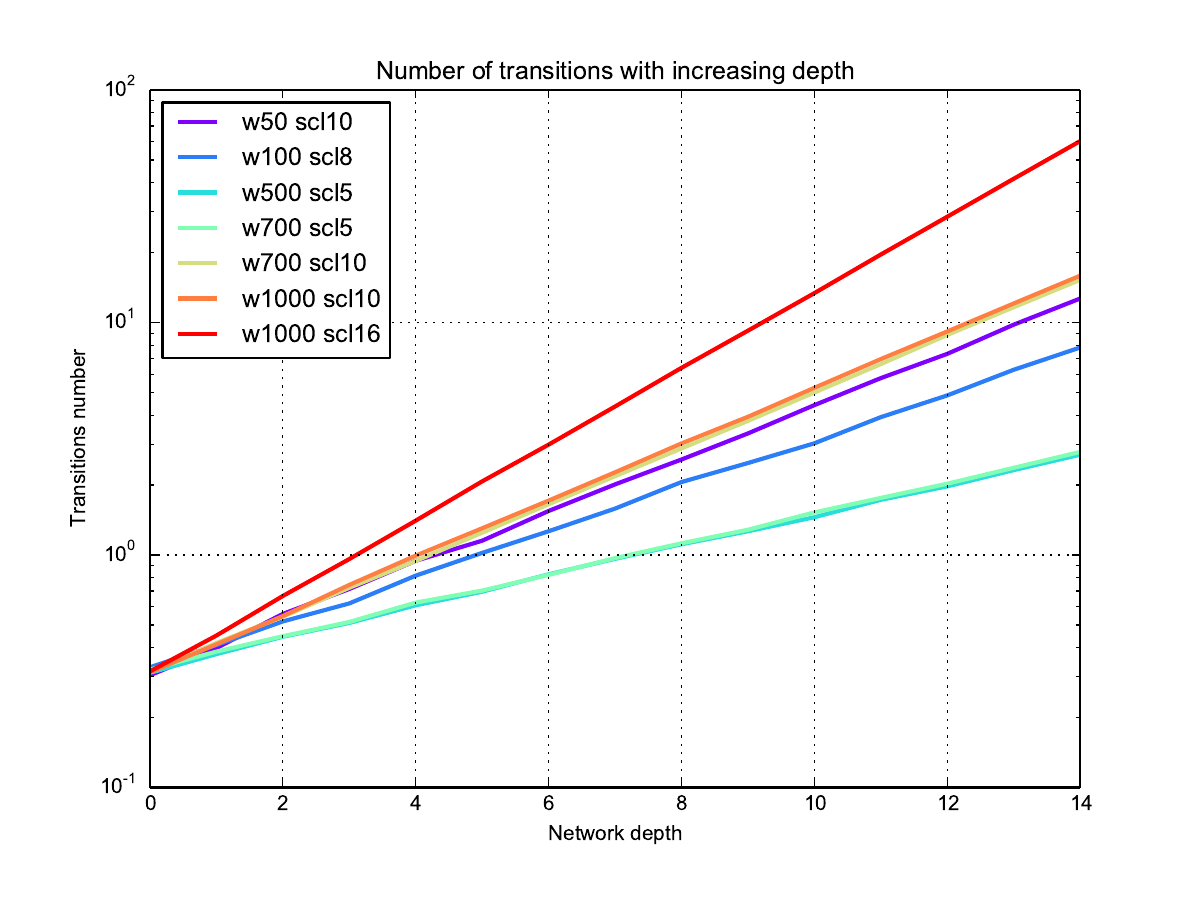} &
\hspace*{-1cm}
\adjincludegraphics[width=0.6\linewidth]{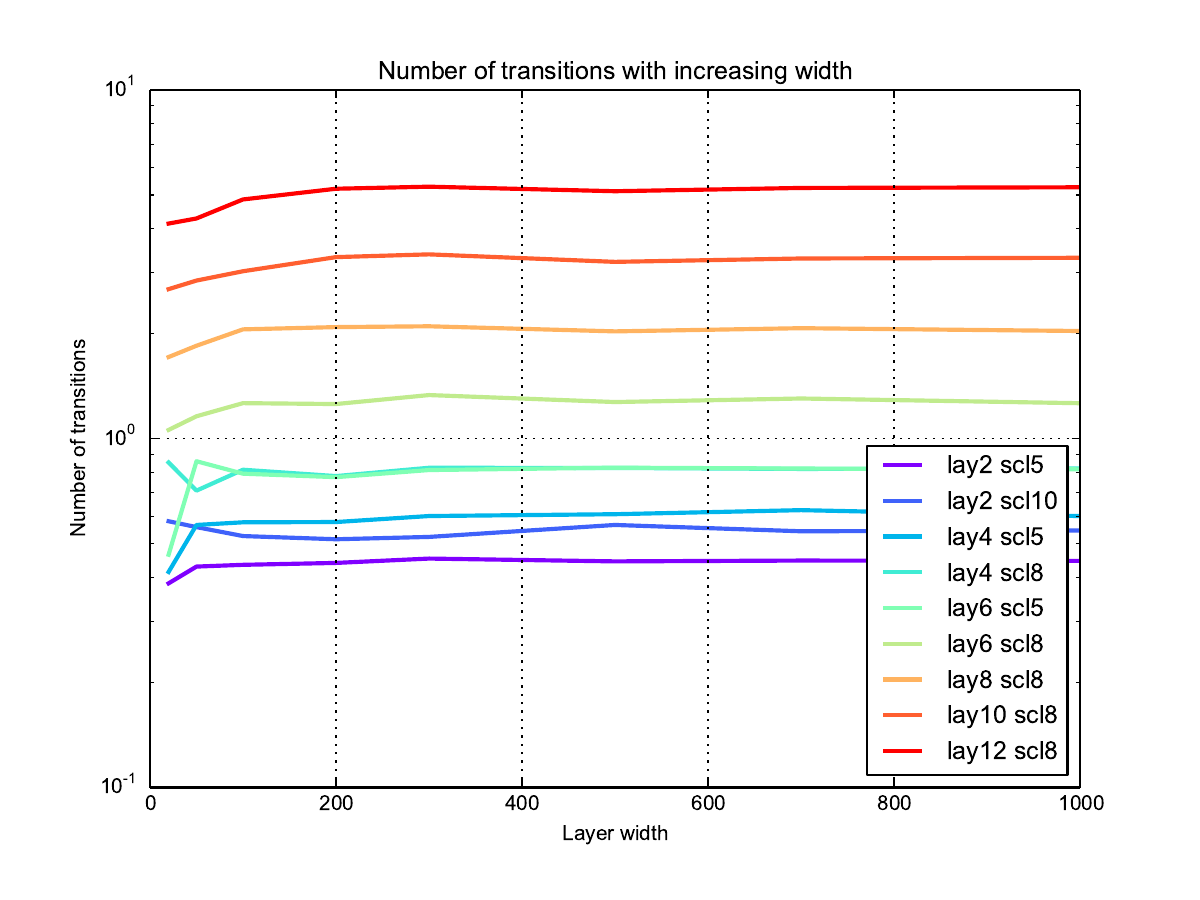}
\end{tabular}
\caption{ \small The number of transitions seen for fully connected networks of different widths, depths and initialization scales, with a circular trajectory between MNIST datapoints. The number of transitions grows exponentially with the depth of the architecture, as seen in (left). The same rate of growth is not seen with increasing architecture width, plotted in (right). There is a surprising dependence on the scale of initialization, explained in \ref{subsec_traj_length}.
}
\label{fig_MNIST_tanh_transitions}
\end{figure}

We empirically tested the growth of the number of activations and transitions as we varied $x$ along $x(t)$ to understand their behavior. We found that for bounded non linearities, especially tanh and hard-tanh, not only do we observe exponential growth with depth (as hinted at by the upper bound) but that the \textit{scale} of parameter initialization also affects the observations (Figure \ref{fig_MNIST_tanh_transitions}). We also experimented with sweeping the \textit{weights} $W$ of a layer through a trajectory $W(t)$, and counting the different labellings output by the network. This `dichotomies' measure is discussed further in the Appendix, and also exhibits the same growth properties, Figure \ref{fig patterns depth and width}.

\subsection{Trajectory Length} \label{subsec_traj_length}

\begin{figure}
\centering
\adjincludegraphics[width=0.9\linewidth]{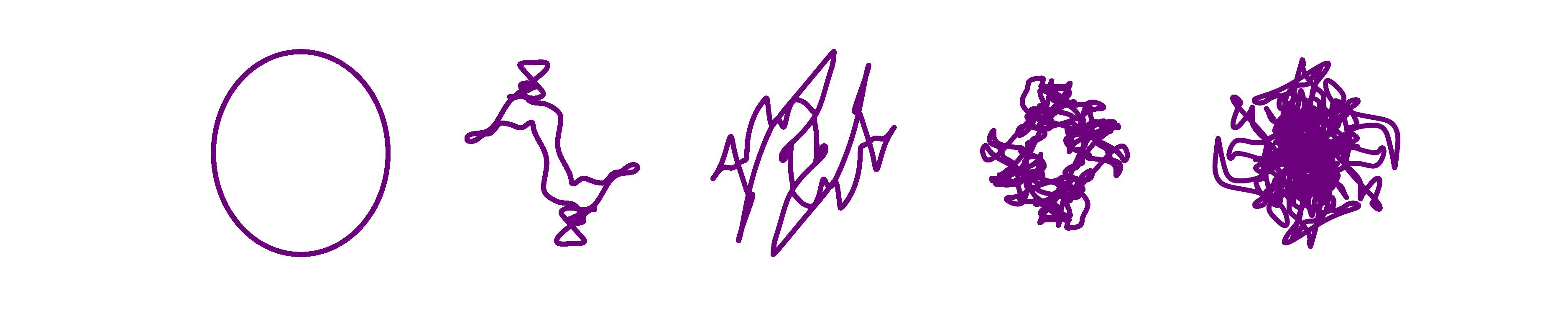}
\caption{ \small Picture showing a trajectory increasing with the depth of a network. We start off with a circular trajectory (left most pane), and feed it through a fully connected tanh network with width $100$. Pane second from left shows the image of the circular trajectory (projected down to two dimensions) after being transformed by the first hidden layer. Subsequent panes show projections of the latent image of the circular trajectory after being transformed by more hidden layers. The final pane shows the the trajectory after being transformed by all the hidden layers.
}
\label{fig_traj_length_increase}
\end{figure}

In fact, there turns out to be a reason for the exponential growth with depth, and the sensitivity to initialization scale. Returning to our definition of trajectory, we can define an immediately related quantity, \textit{trajectory length}

\textbf{Definition:} Given a trajectory, $x(t)$, we define its length, $l(x(t))$, to be the standard \textit{arc length}:
\[ l(x(t)) = \int_t \norm{\frac{d x(t)}{ d t} }dt \]
Intuitively, the arc length breaks $x(t)$ up into infinitesimal intervals and sums together the Euclidean length of these intervals.

If we let $A_{(n,k)}$ denote, as before, fully connected networks with $n$ hidden layers each of width $k$, and initializing with weights $\sim \mathcal{N}(0, \sigma_w^2/k)$ (accounting for input scaling as typical), and biases $\sim \mathcal{N}(0, \sigma_b^2)$, we find that:

\begin{theorem}\label{thm_traj_growth_bias}
\emph{Bound on Growth of Trajectory Length}
\label{thm_lb_perturbation_bias}
Let $F_A(x', W)$ be a ReLU or hard tanh random neural network and $x(t)$ a one dimensional trajectory with $x(t + \delta)$ having a non trival perpendicular component to $x(t)$ for all $t, \delta$ (i.e, not a line). Then defining $z^{(d)}(x(t)) = z^{(d)}(t)$ to be the image of the trajectory in layer $d$ of the network, we have
\begin{itemize}
    \item[(a)] \[   \mbold{E}\left[l(z^{(d)}(t))\right] \geq  O\left(\frac{\sigma_w \sqrt{k}}{\sqrt{k + 1}} \right)^d l(x(t)) \]
    for ReLUs
    
    \item[(b)] \[ \hspace{-0.5cm} \mbold{E}\left[l(z^{(d)}(t))\right] \geq  O\left( \frac{\sigma_w \sqrt{k}}{ \sqrt{ \sigma_w^2 + \sigma_b^2 + k \sqrt{\sigma_w^2 + \sigma_b^2}}} \right)^d l(x(t)) \] for hard tanh
\end{itemize}

\end{theorem} 

That is, $l(x(t)$ grows \textit{exponentially} with the depth of the network, but the width only appears as a base (of the exponent). This bound is in fact \textit{tight} in the limits of large $\sigma_w$ and $k$. 

\begin{figure}
\centering
\adjincludegraphics[width=0.7\linewidth]{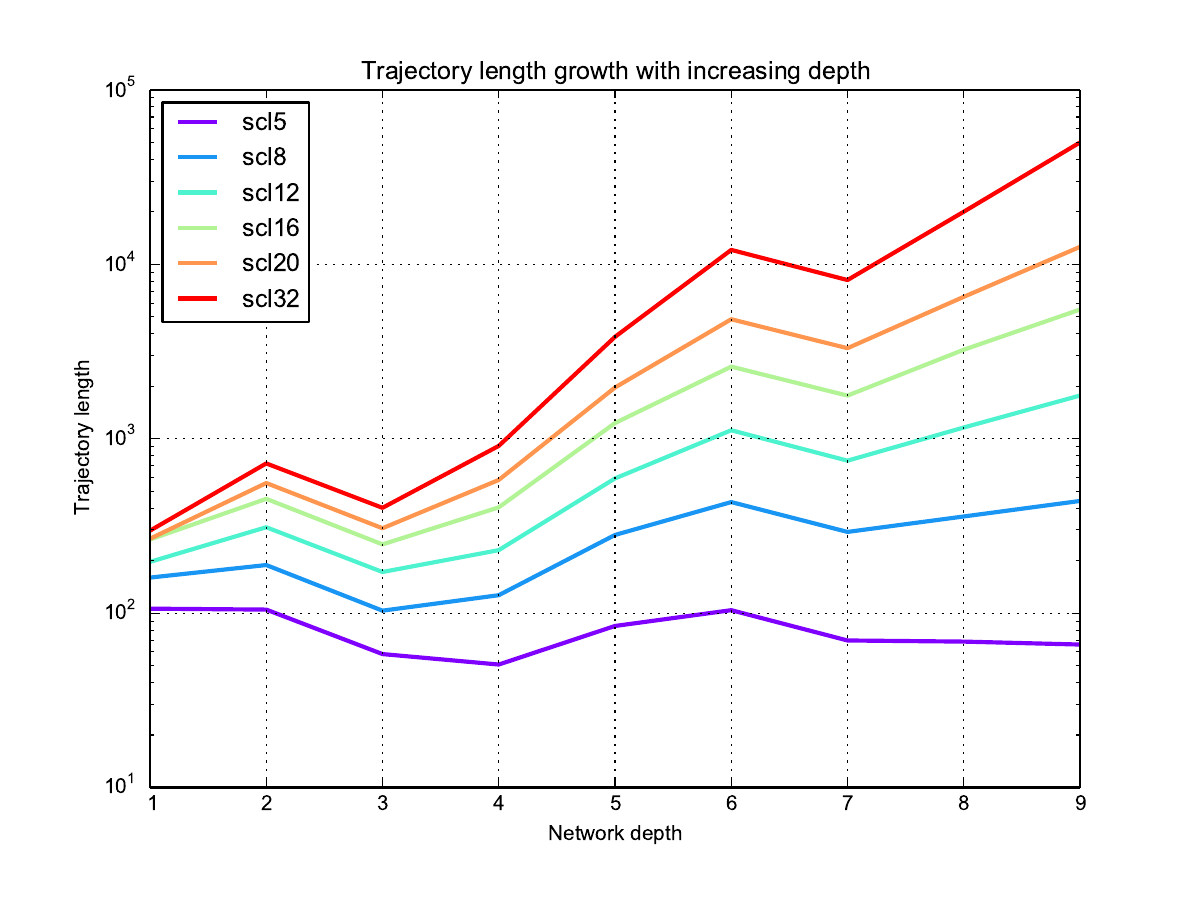}
\caption{\small We look at trajectory growth with different initialization scales as a trajectory is propagated through a convolutional architecture for CIFAR-10, with ReLU activations. The analysis of Theorem \ref{thm_traj_growth_bias} was for fully connected networks, but we see that trajectory growth holds (albeit with slightly higher scales) for convolutional architectures also. Note that the decrease in trajectory length, seen in layers $3$ and $7$ is expected, as those layers are pooling layers.
}
\label{CIFAR10_traj_length}
\end{figure}

A schematic image depicting this can be seen in Figure \ref{fig_traj_length_increase} and the proof can be found in the Appendix. A rough outline is as follows: we look at the expected growth of the difference between a point $z^{(d)}(t)$ on the curve and a small perturbation $z^{(d)}(t + d t)$, from layer $d$ to layer $d+1$. Denoting this quantity $\norm{\delta z^{(d)}(t)}$, we derive a recurrence relating $\norm{\delta z^{(d+1)}(t)}$ and $\norm{\delta z^{(d)}(t)}$ which can be composed to give the desired growth rate.

The analysis is complicated by the statistical dependence on the image of the input $z^{(d+1)}(t)$. So we instead form a recursion by looking at the component of the difference perpendicular to the image of the input in that layer, i.e. $\norm{\perpart{\delta z^{(d+1)}}(t)}$, which results in the condition on $x(t)$ in the statement.

In Figures \ref{CIFAR10_traj_length}, \ref{MNIST_traj_length}, we see the growth of an input trajectory for ReLU  networks on CIFAR-10 and MNIST. The CIFAR-10 network is  convolutional but we observe that these layers also result in similar rates of trajectory length increases to the fully connected layers. We also see, as would be expected, that pooling layers act to \textit{reduce} the trajectory length. We discuss upper bounds in the Appendix.

\begin{figure}[h]
\centering
\adjincludegraphics[width=0.6\linewidth]{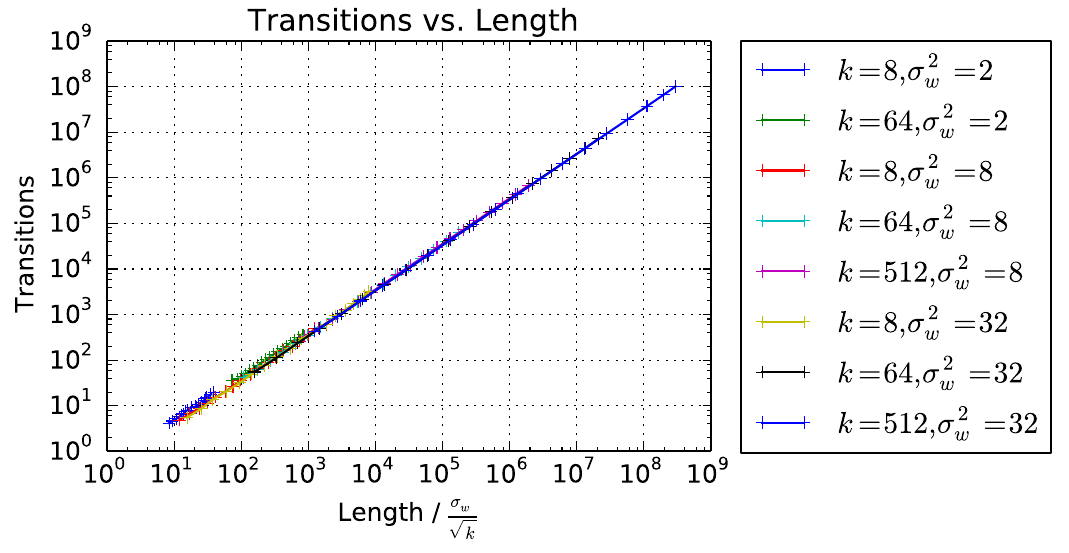}
\caption{
\small The number of transitions is linear in trajectory length.
Here we compare the empirical number of transitions to the length of the trajectory,
for different depths of a hard-tanh network. 
We repeat this comparison for a variety of network architectures, with different network width $k$ 
and weight variance $\sigma^2_w$.
\label{fig linear transitions}
}
\end{figure}

For the hard tanh case (and more generally any bounded non-linearity), we can formally prove the relation of trajectory length and transitions under an assumption: assume that while we sweep $x(t)$ all neurons are saturated unless transitioning saturation endpoints, which happens very rapidly. (This is the case for e.g. large initialization scales). Then we have:

\begin{theorem}\label{thm_large_weight_limit}
\emph{Transitions proportional to trajectory length}
Let $F_{A_{n,k}}$ be a hard tanh network with $n$ hidden layers each of width $k$. And let
\[ g(k, \sigma_w, \sigma_b, n) = O\left(\frac{\sqrt{k}}{\sqrt{1 + \frac{\sigma_b^2}{\sigma_w^2}}}\right)^n \]
Then $\mathcal{T}(F_{A_{n,k}}(x(t); W) = O(g(k, \sigma_w, \sigma_b, n))$ for $W$ initialized with weight and bias scales $\sigma_w, \sigma_b$.
\end{theorem} 

Note that the expression for $ g(k, \sigma_w, \sigma_b, n)$ is exactly the expression given by Theorem \ref{thm_lb_perturbation_bias} when $\sigma_w$ is very large and dominates $\sigma_b$. We can also verify this experimentally in settings where the simpilfying assumption does not hold, as in Figure \ref{fig linear transitions}.

\section{Insights from Network Expressivity} \label{experiments}
Here we explore the insights gained from applying our measurements of expressivity, particularly trajectory length, to understand network performance. We examine the connection of expressivity and stability, and inspired by this, propose a new method of regularization, \textit{trajectory regularization} that offers the same advantages as the more computationally intensive batch normalization.
\subsection{Expressivity and Network Stability}

The analysis of network expressivity offers interesting takeaways related to the parameter and functional stability of a network. From the proof of Theorem \ref{thm_lb_perturbation_bias}, we saw that a perturbation to the input would grow exponentially in the depth of the network. It is easy to see that this analysis is not limited to the input layer, but can be applied to \textit{any} layer. In this form, it would say

\begin{center}
A perturbation at a layer grows exponentially in the \textit{remaining depth} after that layer.
\end{center}
\begin{figure}
\centering
\adjincludegraphics[width=0.5\linewidth]{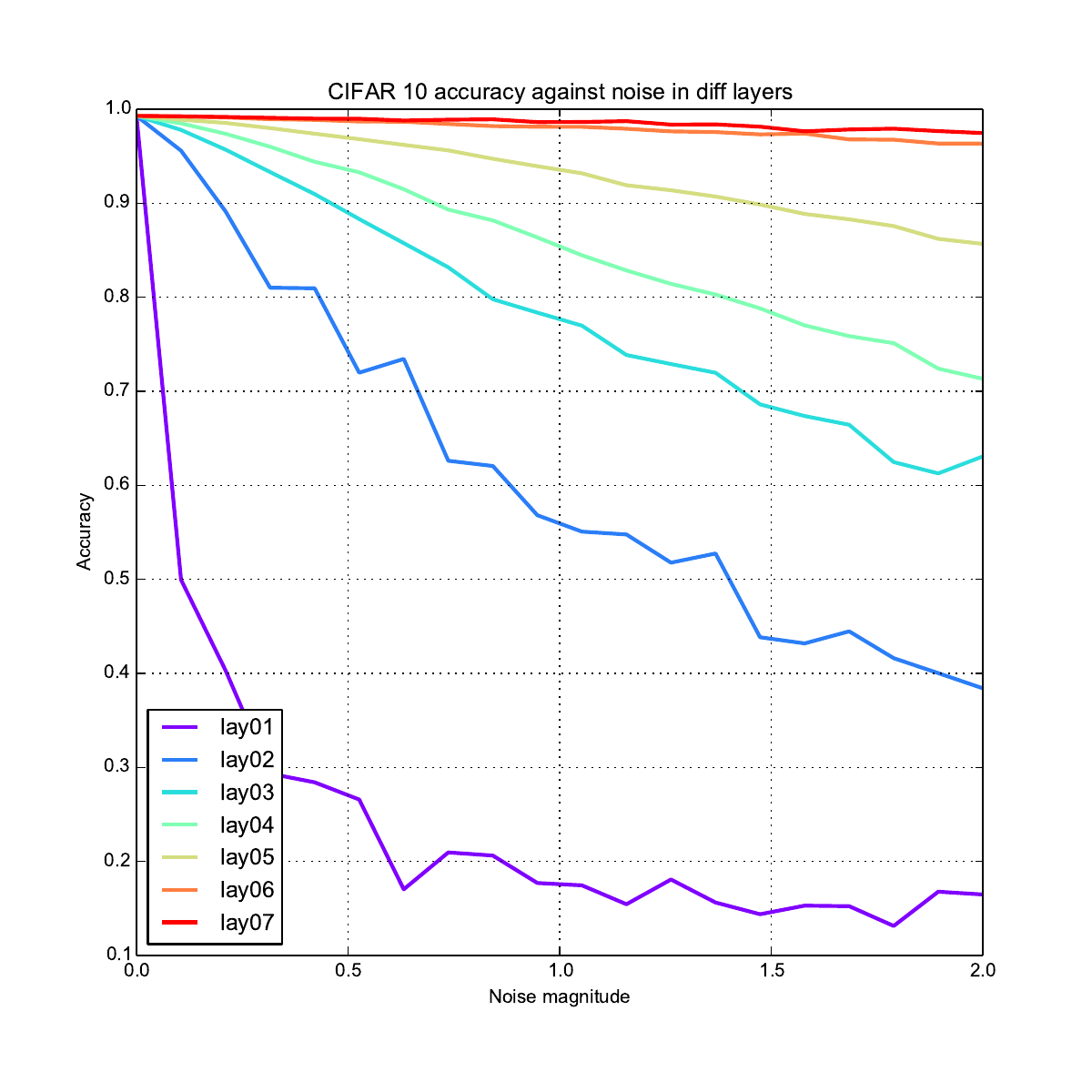}
\caption{\small
 We then pick a single layer of a conv net trained to high accuracy on CIFAR10, and add noise to the layer weights of increasing magnitudes, testing the network accuracy as we do so. We find that the initial (lower) layers of the network are \textit{least} robust to noise -- as the figure shows, adding noise of $0.25$ magnitude to the first layer results in a $0.7$ drop in accuracy, while the same amount of noise added to the fifth layer barely results in a $0.02$ drop in accuracy. This pattern is seen for many different initialization scales, even for a (typical) scaling of $\sigma_w^2 = 2$, used in the experiment.
\label{fig_CIFAR_noise}
}
\end{figure}

This means that perturbations to weights in lower layers should be more costly than perturbations in the upper layers, due to exponentially increasing magnitude of noise, and result in a much larger drop of accuracy. Figure \ref{fig_CIFAR_noise}, in which we train a conv network on CIFAR-10 and add noise of varying magnitudes to exactly one layer, shows exactly this. 

We also find that the converse (in some sense) holds: after initializing a network, we trained a single layer at different depths in the network and found monotonically increasing performance as layers lower in the network were trained. This is shown in Figure \ref{mnist_ffn} and  Figure \ref{cifar_ffn} in the Appendix.

\begin{figure}[h]
\centering
\adjincludegraphics[width=0.75\linewidth]{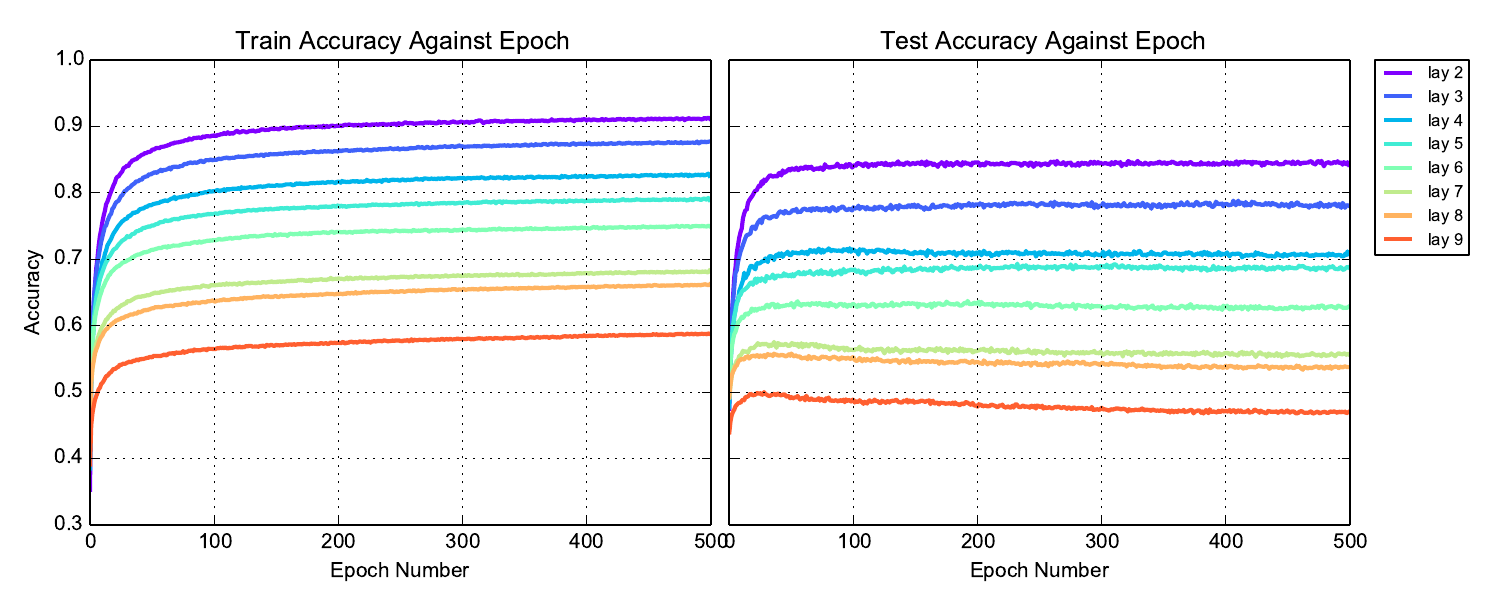}
\caption{
\small
Demonstration of expressive power of remaining depth on MNIST.
Here we plot train and test accuracy achieved by training exactly one layer of a fully connected neural net on MNIST. The different lines are generated by varying the hidden layer chosen to train. All other layers are kept frozen after random initialization. We see that training lower hidden layers leads to better performance. The networks had width $k=100$, weight variance $\sigma_w^2 = 1$, and hard-tanh nonlinearities. Note that we only train from the second hidden layer (weights $W^{(1)}$) onwards, so that the number of parameters trained remains fixed.
\label{mnist_ffn}
}
\end{figure}

\subsection{Trajectory Length and Regularization: The Effect of Batch Normalization}
Expressivity measures, especially trajectory length, can also be used to better understand the effect of regularization. One regularization technique that has been extremely successful for training neural networks is Batch Normalization \cite{ioffe2015batch}.

\begin{figure}[h]
\centering
\adjincludegraphics[width=0.6\linewidth]{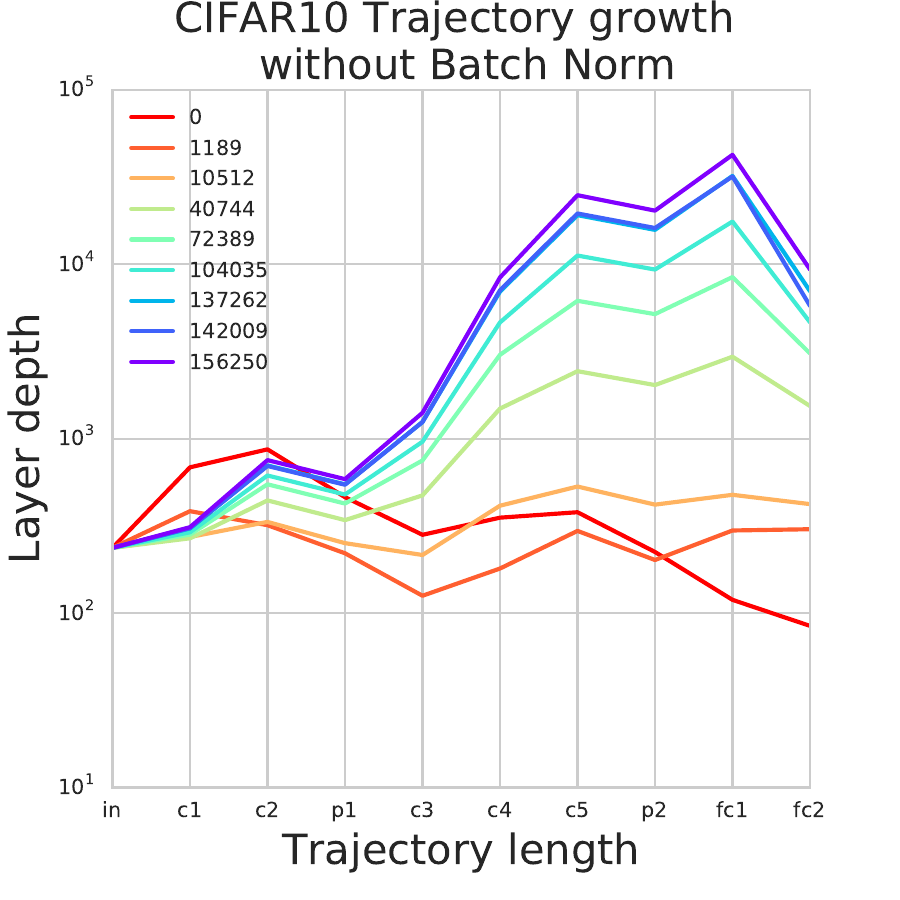}
\caption{\small Training increases trajectory length even for typical ($\sigma_w^2 = 2$) initialization values of $\sigma_w$. Here we propagate a circular trajectory joining two CIFAR10 datapoints through a conv net without batch norm, and look at how trajectory length changes through training. We see that training causes trajectory length to increase exponentially with depth (exceptions only being the pooling layers and the final fc layer, which halves the number of neurons.) Note that at Step $0$, the network is \textit{not} in the exponential growth regime. We observe (discussed in Figure \ref{fig_CIFAR_weight_norms}) that even networks that aren't initialized in the exponential growth regime can be pushed there through training.
\label{CIFAR10_traj_nobn}
}
\end{figure}

By taking measures of trajectories during training we find that without batch norm, trajectory length tends to increase during training, as shown in Figures \ref{CIFAR10_traj_nobn} and Figure \ref{mnist_traj_trans_sigma_3} in the Appendix. In these experiments, two networks were initialized with $\sigma_w^2 = 2$ and trained to high test accuracy on CIFAR10 and MNIST. We see that in both cases, trajectory length increases as training progresses.

A surprising observation is $\sigma_w^2 = 2$ is not in the exponential growth increase regime at initialization for the CIFAR10 architecture (Figure \ref{CIFAR10_traj_nobn} at Step $0$.). But note that even with a smaller weight initialization, weight norms increase during training, shown in Figure \ref{fig_CIFAR_weight_norms}, pushing typically initialized networks into the exponential growth regime.

\begin{figure}[h]
\centering
\adjincludegraphics[width=0.6\linewidth]{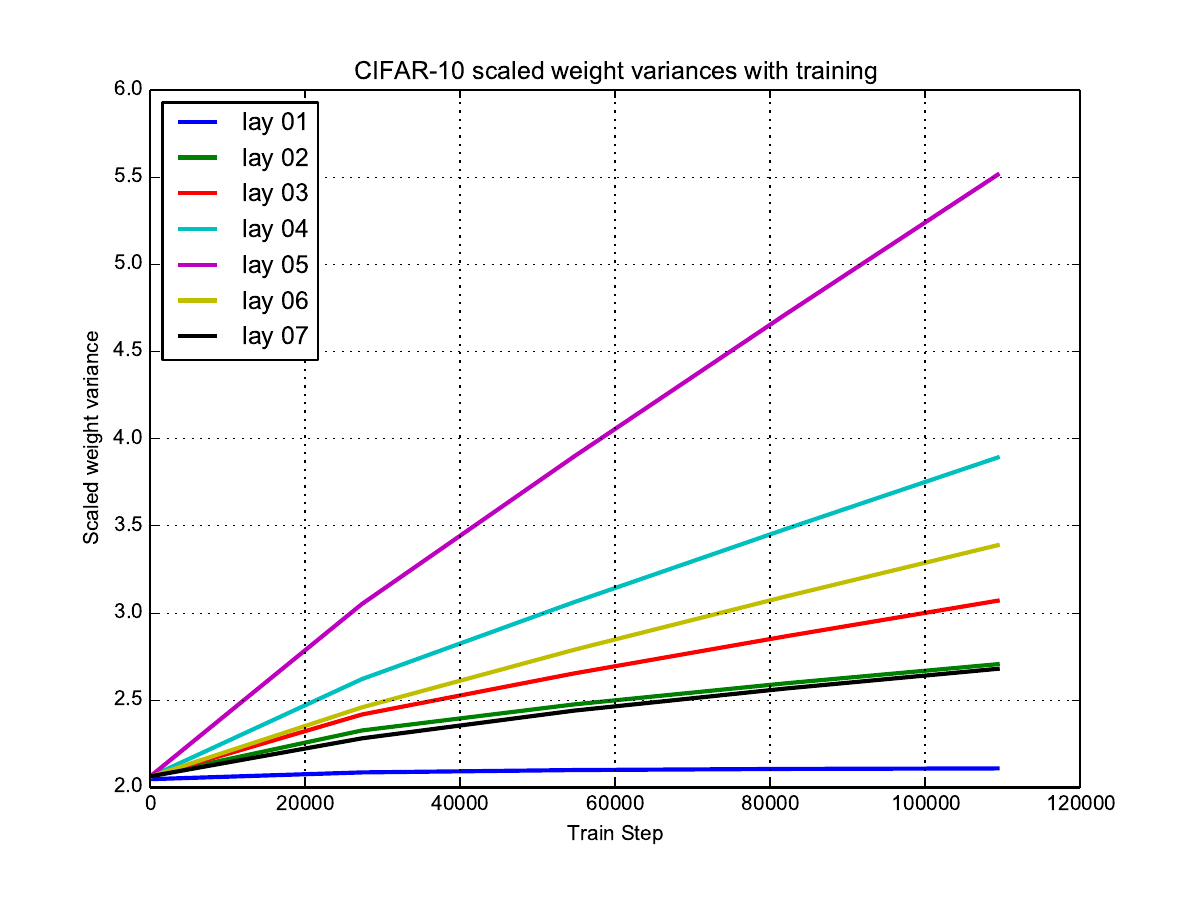}
\caption{ \small
This figure shows how the weight scaling of a CIFAR10 network evolves during training. The network was initialized with $\sigma_w^2 = 2$, which increases across all layers during training.
\label{fig_CIFAR_weight_norms}
}
\end{figure}

While the initial growth of trajectory length enables greater functional expressivity, large trajectory growth in the learnt representation results in an unstable representation, witnessed in Figure \ref{fig_CIFAR_noise}.  In Figure \ref{fig_CIFAR_bn} we train another conv net on CIFAR10, but this time with batch normalization. We see that the batch norm layers reduce trajectory length, helping stability. 

\begin{figure}[h]
\centering
\adjincludegraphics[width=0.7\linewidth]{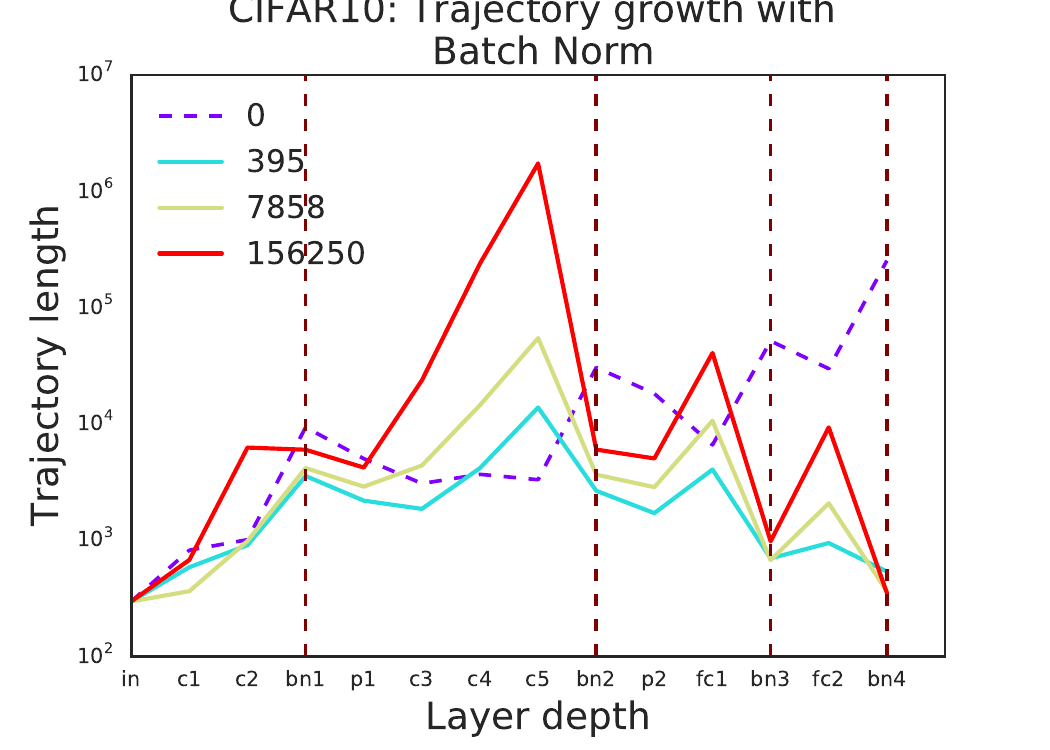}
\caption{Growth of circular trajectory between two datapoints with batch norm layers for a conv net on CIFAR10. The network was initialized as typical, with $\sigma_w^2 = 2$. Note that the batch norm layers in Step $0$ are poorly behaved due to division by a close to $0$ variance. But after just a few hundred gradient steps and continuing onwards, we see the batch norm layers (dotted lines) reduce trajectory length, stabilising the representation without sacrificing expressivity.
\label{fig_CIFAR_bn}
}
\end{figure}

\subsection{Trajectory Regularization}
\label{sec_traj_reg}
Motivated by the fact that batch normalization decreases trajectory length and hence helps stability and generalization, we consider directly regularizing on trajectory length: we replace every batch norm layer used in the conv net in Figure \ref{fig_CIFAR_bn} with a \textit{trajectory regularization layer}. This layer adds to the loss $\lambda(\text{current length}/\text{orig length})$, and then scales the outgoing activations by $\lambda$, where $\lambda$ is a parameter to be learnt. In implementation, we typically scale the additional loss above with a constant ($0.01$) to reduce magnitude in comparison to classification loss. Our results, Figure \ref{CIFAR_traj_reg} show that both trajectory regularization and batch norm perform comparably, and considerably better than not using batch norm. One advantage of using Trajectory Regularization is that we don't require different computations to be performed for train and test, enabling more efficient implementation.

\begin{figure}[h]
\centering
\adjincludegraphics[width=0.7\linewidth]{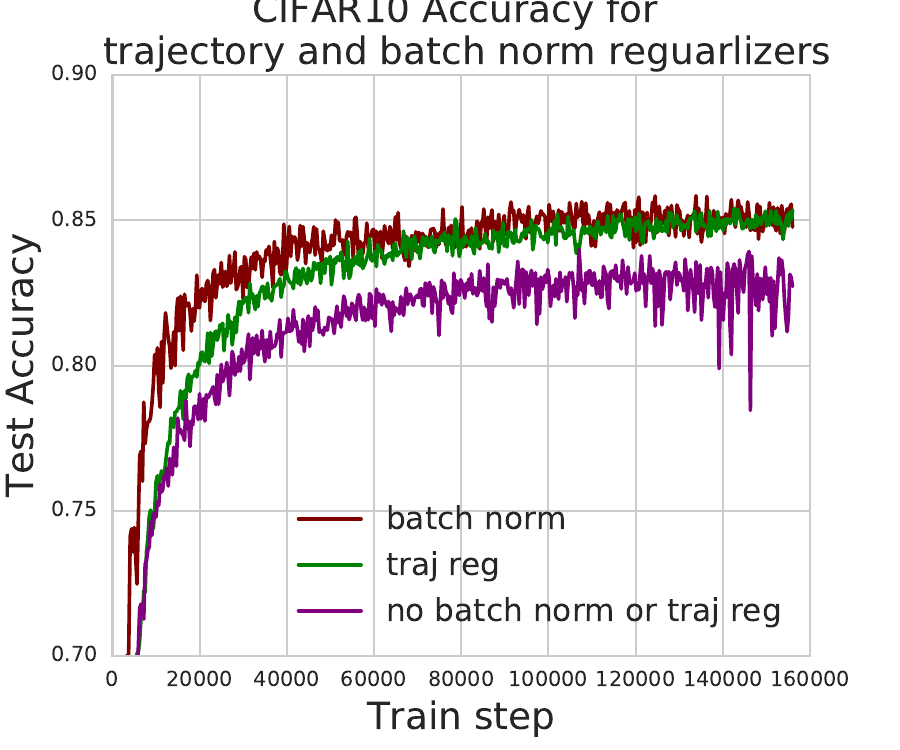}
\caption{ \small
We replace each batch norm layer of the CIFAR10 conv net with a \textit{trajectory regularization} layer, described in Section \ref{sec_traj_reg}. During training trajectory length is easily calculated as a piecewise linear trajectory between adjacent datapoints in the minibatch. We see that trajectory regularization achieves the same performance as batch norm, albeit with slightly more train time. However, as trajectory regularization behaves the same during train and test time, it is simpler and more efficient to implement.
\label{CIFAR_traj_reg}
}
\end{figure}

\section{Discussion}
Characterizing the expressiveness of neural networks, and understanding how expressiveness varies with parameters of the architecture, has been a challenging problem due to the difficulty in identifying meaningful notions of expressivity and in linking their analysis to implications for these networks in practice. In this paper we have presented an interrelated set of expressivity measures; we have shown tight exponential bounds on the growth of these measures in the depth of the networks, and we have offered a unifying view of the analysis through the notion of {\em trajectory length}. Our analysis of trajectories provides insights for the performance of trained networks as well, suggesting that networks in practice may be more sensitive to small perturbations in weights at lower layers. We also used this to explore the empirical success of batch norm, and developed a new regularization method -- trajectory regularization.

This work raises many interesting directions for future work. At a general level, continuing the theme of `principled deep understanding', it would be interesting to link measures of expressivity to other properties of neural network performance. There is also a natural connection between adversarial examples, \cite{goodfellow2014explaining}, and trajectory length: adversarial perturbations are only a small distance away in input space, but result in a large change in classification (the output layer). Understanding how trajectories between the original input and an adversarial perturbation grow might provide insights into this phenomenon. Another direction, partially explored in this paper, is regularizing based on trajectory length. A very simple version of this was presented, but further performance gains might be achieved through more sophisticated use of this method.

\section*{Acknowledgements}

We thank Samy Bengio, Ian Goodfellow, Laurent Dinh, and Quoc Le for extremely helpful discussion. 


\setlength{\bibsep}{0pt plus 0.2ex}

\bibliographystyle{unsrtnat}

\bibliography{deep_expressivity}

\clearpage

\onecolumn
\appendix

\normalsize

\part*{Appendix}

Here we include the full proofs from sections in the paper.

\section{Proofs and additional results from Section \ref{subsec_hyperplanes}}

\paragraph{Proof of Theorem \ref{thm_carve_inp_space}}

\begin{proof}
We show inductively that $F_A(x;W)$ partitions the input space into convex polytopes via hyperplanes. Consider the image of the input space under the first hidden layer. Each neuron $\node{1}{i}$ defines hyperplane(s) on the input space: letting $W^{(0)}_i$ be the $i$th row of $W^{(0)}$, $b^{(0)}_i$ the bias, we have the hyperplane $W^{(0)}_i x + b_i = 0$ for a ReLU and hyperplanes $W^{(0)}_i x + b_i = \pm 1$ for a hard-tanh. Considering all such hyperplanes over neurons in the first layer, we get a hyperplane arrangement in the input space, each polytope corresponding to a specific activation pattern in the first hidden layer. 

Now, assume we have partitioned our input space into convex polytopes with hyperplanes from layers $\leq d - 1$. Consider $\node{d}{i}$ and a specific polytope $R_i$. Then the activation pattern on layers $\leq d - 1$ is constant on $R_i$, and so the input to $\node{d}{i}$ on $R_i$ is a linear function of the inputs $\sum_j \lambda_j x_j + b$ and some constant term, comprising of the bias and the output of saturated units. Setting this expression to zero (for ReLUs) or to $\pm 1$ (for hard-tanh) again gives a hyperplane equation, but this time, the equation is only valid in $R_i$ (as we get a different linear function of the inputs in a different region.) So the defined hyperplane(s) either partition $R_i$ (if they intersect $R_i$) or the output pattern of $\node{d}{i}$ is also constant on $R_i$. The theorem then follows.
\end{proof}

This implies that any one dimensional trajectory $x(t)$, that does not `double back' on itself (i.e. reenter a polytope it has previously passed through), will not repeat activation patterns. In particular, after seeing a transition (crossing a hyperplane to a different region in input space) we will never return to the region we left. A simple example of such a trajectory is a straight line:
\begin{corollary}
\label{corr_affine}
\emph{Transitions and Output Patterns in an Affine Trajectory}
For any affine one dimensional trajectory $x(t) = x_0 + t(x_1 - x_0)$ input into a neural network $F_W$, we partition $\mbold{R} \ni t$ into intervals every time a neuron transitions. Every interval has a \textit{unique} network activation pattern on $F_W$.
\end{corollary}

Generalizing from a one dimensional trajectory, we can ask how many regions are achieved over the entire input -- i.e. how many distinct activation patterns are seen? We first prove a bound on the number of regions formed by $k$ hyperplanes in $\mbold{R}^m$ (in a purely elementary fashion, unlike the proof presented in \citep{stanleyhyperplanes})

\begin{theorem}
\label{prop-num-regions}
\emph{Upper Bound on Regions in a Hyperplane Arrangement} Suppose we have $k$ hyperplanes in $\mbold{R}^m$ - i.e. $k$ equations of form $\alpha_i x = \beta_i$. for $\alpha_i \in \mbold{R}^m$, $\beta_i \in \mbold{R}$. Let the number of regions (connected open sets bounded on some sides by the hyperplanes) be $r(k,m)$. Then
\[ r(k,m) \leq \sum_{i=0}^m \binom{k}{i} \]
\end{theorem}

\paragraph{Proof of Theorem \ref{prop-num-regions}}

\begin{proof}
Let the hyperplane arrangement be denoted $\mathcal{H}$, and let $H \in \mathcal{H}$ be one specific hyperplane. Then the number of regions in $\mathcal{H}$ is precisely the number of regions in $\mathcal{H} - H$ plus the number of regions in $\mathcal{H} \cap H$. (This follows from the fact that $H$ subdivides into two regions exactly all of the regions in $\mathcal{H} \cap H$, and does not affect any of the other regions.) 

In particular, we have the recursive formula
\[ r(k, m) = r(k -1, m) + r(k - 1, m - 1) \]

We now induct on $k + m$ to assert the claim. The base cases of $r(1,0) = r(0,1) = 1$ are trivial, and assuming the claim for  $ \leq k + m - 1$ as the induction hypothesis, we have
\begin{align*}
r(k - 1, m) + r(k - 1, m - 1) & \leq \sum_{i=0}^m \binom{k-1}{i} + \sum_{i=0}^{m-1} \binom{k-1}{i} \\
& \leq \binom{k-1}{0} + \sum_{i=0}^{d-1} \binom{k-1}{i} + \binom{k-1}{i+1} \\
& \leq \binom{k}{0} + \sum_{i=0}^{m - 1} \binom{k}{i + 1}
\end{align*}
where the last equality follows by the well known identity 
\[ \binom{a}{b} + \binom{a}{b+1} = \binom{a+1}{b+1} \]

This concludes the proof.
\end{proof}

With this result, we can easily prove Theorem \ref{thm act count} as follows:

\paragraph{Proof of Theorem \ref{thm act count}}

\begin{proof}
First consider the ReLU case. Each neuron has one hyperplane associated with it, and so by Theorem \ref{prop-num-regions}, the first hidden layer divides up the inputs space into $r(k,m)$ regions, with $r(k,m) \leq O(k^m)$. 

Now consider the second hidden layer. For every region in the first hidden layer, there is a different activation pattern in the first layer, and so (as described in the proof of Theorem \ref{thm_carve_inp_space}) a different hyperplane arrangement of $k$ hyperplanes in an $m$ dimensional space, contributing at most $r(k,m)$ regions.

In particular, the total number of regions in input space as a result of the first and second hidden layers is $\leq r(k,m)*r(k,m) \leq O(k^2m)$. Continuing in this way for each of the $n$ hidden layers gives the $O(k^mn)$ bound. 

A very similar method works for hard tanh, but here each neuron produces two hyperplanes, resulting in a bound of $O((2k)^{mn})$.

\end{proof}

\section{Proofs and additional results from Section \ref{subsec_traj_length}}

\paragraph{Proof of Theorem \ref{thm_traj_growth_bias}}

\subsection{Notation and Preliminary Results}

\textit{Difference of points on trajectory} Given $x(t) = x, x(t + d t) = x + \delta x$ in the trajectory, let $\delta z^{(d)} = z^{(d)}(x + \delta x) - z^{(d)}(x)$ 

\textit{Parallel and Perpendicular Components:} Given vectors $x, y$, we can write $y = \perpart{y} + \parpart{y} $ where $\perpart{y}$ is the component of $y$ perpendicular to $x$, and $\parpart{y}$ is the component parallel to $x$. (Strictly speaking, these components should also have a subscript $x$, but we suppress it as the direction with respect to which parallel and perpendicular components are being taken will be explicitly stated.)

This notation can also be used with a matrix $W$, see Lemma \ref{lemma_matrix_decomp}.

Before stating and proving the main theorem, we need a few preliminary results.

\begin{lemma}
\label{lemma_matrix_decomp}
\emph{Matrix Decomposition}
Let $x, y \in \mbold{R}^k$ be fixed non-zero vectors, and let $W$ be a (full rank) matrix. Then, we can write
\[ W = \parpar{W} + \parperp{W} + \perpar{W} + \perperp{W}  \]
such that
\begin{align*}
    \parperp{W} x = 0 && \perperp{W} x = 0 \\
    y^T \perpar{W} = 0 && y^T \perperp{W}  = 0
\end{align*}
i.e. the row space of $W$ is decomposed to perpendicular and parallel components with respect to $x$ (subscript on right), and the column space is decomposed to perpendicular and parallel components of $y$ (superscript on left).
\end{lemma}

\begin{proof}
Let $V, U$ be rotations such that $Vx = (\norm{x}, 0... ,0)^T $ and $Uy = (\norm{y}, 0 ... 0)^T$. Now let $\tilde{W} = UWV^T$, and let $\tilde{W} = \parpar{\tilde{W}} + \parperp{\tilde{W}} + \perpar{\tilde{W}} + \perperp{\tilde{W}}$, with $\parpar{\tilde{W}}$ having non-zero term exactly $\tilde{W}_{11}$, $\parperp{\tilde{W}}$ having non-zero entries exactly $\tilde{W}_{1i}$ for $2 \leq i \leq k$. Finally, we let $\perpar{\tilde{W}}$ have non-zero entries exactly $\tilde{W}_{i1}$, with $ 2 \leq i \leq k$ and $\perperp{\tilde{W}}$ have the remaining entries non-zero.

If we define $\tilde{x} = Vx$ and $\tilde{y} = Uy$, then we see that
\begin{align*}
    \parperp{\tilde{W}} \tilde{x} = 0 && \perperp{\tilde{W}} \tilde{x} = 0 \\
    \tilde{y}^T \perpar{\tilde{W}} = 0 && \tilde{y}^T \perperp{\tilde{W}} = 0 
\end{align*}
as $\tilde{x}, \tilde{y}$ have only one non-zero term, which does not correspond to a non-zero term in the components of $\tilde{W}$ in the equations.

Then, defining $\parpar{W} = U^T \parpar{\tilde{W}} V$, and the other components analogously, we get equations of the form
\[\parperp{W} x = U^T \parperp{\tilde{W}} V x = U^T \parperp{\tilde{W}} \tilde{x} = 0\]
\end{proof}

\begin{obs}
Given $W, x$ as before, and considering $\parpart{W}$, $\perpart{W}$ with respect to $x$ (wlog a unit vector) we can express them directly in terms of $W$ as follows: Letting $W^{(i)}$ be the $i$th row of $W$, we have

\[ \parpart{W} = \begin{pmatrix} ((W^{(0)})^T \cdot x) x \\ \vdots \\ ((W^{(k)})^T \cdot x) x \end{pmatrix} \]
i.e. the projection of each row in the direction of $x$. And of course
\[ \perpart{W} = W - \parpart{W} \]
\end{obs}

The motivation to consider such a decomposition of $W$ is for the resulting independence between different components, as shown in the following lemma.

\begin{lemma}
\emph{Independence of Projections}
\label{lemma_indep_proj}
Let $x$ be a given vector (wlog of unit norm.) If $W$ is a random matrix with $W_{ij} \sim \mathcal{N}(0, \sigma^2)$, then $\parpart{W}$ and $\perpart{W}$ with respect to $x$ are independent random variables.
\end{lemma}

\begin{proof}

There are two possible proof methods:
\begin{enumerate}
\item[(a)] We use the rotational invariance of random Gaussian matrices, i.e. if $W$ is a Gaussian matrix, iid entries $\mathcal{N}(0, \sigma^2)$, and $R$ is a rotation, then $RW$ is also iid Gaussian, entries $\mathcal{N}(0, \sigma^2)$. (This follows easily from affine transformation rules for multivariate Gaussians.)

Let $V$ be a rotation as in Lemma \ref{lemma_matrix_decomp}. Then $\tilde{W} = WV^T$ is also iid Gaussian, and furthermore, $\parpart{\tilde{W}}$ and $\perpart{\tilde{W}}$ partition the entries of $\tilde{W}$, so are evidently independent. But then $\parpart{W} = \parpart{\tilde{W}}V^T$ and $\perpart{W} = \perpart{\tilde{W}}V^T$ are also independent.

\item[(b)]From the observation note that $\parpart{W}$ and $\perpart{W}$ have a centered multivariate joint Gaussian distribution (both consist of linear combinations of the entries $W_{ij}$ in $W$.) So it suffices to show that $\parpart{W}$ and $\perpart{W}$ have covariance $0$. Because both are centered Gaussians, this is equivalent to showing $\mbold{E}(<\parpart{W}, \perpart{W}>) = 0$.  We have that
\[ \mbold{E}(<\parpart{W}, \perpart{W}>) = \mbold{E}( \parpart{W} \perpart{W}^T) = \mbold{E}(\parpart{W} W^T) - \mbold{E}(\parpart{W} \parpart{W}^T) \]
As any two rows of $W$ are independent, we see from the observation that $\mbold{E}(\parpart{W} W^T)$ is a diagonal matrix, with the $i$th diagonal entry just $((W^{(0)})^T \cdot x)^2$. But similarly, $\mbold{E}(\parpart{W} \parpart{W}^T)$ is also a diagonal matrix, with the same diagonal entries - so the claim follows.
\end{enumerate}
\end{proof}

In the following two lemmas, we use the rotational invariance of Gaussians as well as the chi distribution to prove results about the expected norm of a random Gaussian vector.

\begin{lemma}\label{lemma_norm}
\emph{Norm of a Gaussian vector}
Let $X \in \mbold{R}^k $ be a random Gaussian vector, with $X_i$ iid,  $\sim \mathcal{N}(0, \sigma^2)$. Then 
\[ \mbold{E} \left[ \norm{X} \right] = \sigma \sqrt{2} \frac{\Gamma((k+1)/2)}{\Gamma(k/2)} \]
\end{lemma}

\begin{proof}
We use the fact that if $Y$ is a random Gaussian, and $Y_i \sim \mathcal{N}(0, 1)$ then $\norm{Y}$ follows a chi distribution. This means that $\mbold{E}(\norm{X/\sigma}) = \sqrt{2}\Gamma((k+1)/2)/\Gamma(k/2)$, the mean of a chi distribution with $k$ degrees of freedom, and the result follows by noting that the expectation in the lemma is $\sigma$ multiplied by the above expectation.
\end{proof}

We will find it useful to bound ratios of the Gamma function (as appear in Lemma \ref{lemma_norm}) and so introduce the following inequality, from \citep{extendedgautschi} that provides an extension of Gautschi's Inequality.

\begin{theorem}\label{thm_gautschi}
\emph{An Extension of Gautschi's Inequality}
For $0 < s < 1$, we have
\[ \left( x + \frac{s}{2} \right)^{1-s} \leq \frac{\Gamma(x + 1)}{\Gamma(x + s)} \leq \left(x - \frac{1}{2} + \left(s + \frac{1}{4} \right)^{\frac{1}{2}} \right)^{1 -s} \]
\end{theorem}

We now show:

\begin{lemma}\label{norm_projections}
\emph{Norm of Projections}
Let $W$ be a $k$ by $k$ random Gaussian matrix with iid entries $\sim \mathcal{N}(0, \sigma^2)$, and $x, y$ two given vectors. Partition $W$ into components as in Lemma \ref{lemma_matrix_decomp} and let $\perpart{x}$ be a nonzero vector perpendicular to $x$. Then 
\begin{enumerate}
  \item[(a)] \[\mbold{E}\left[\norm{\perperp{W}\perpart{x}} \right] = \norm{\perpart{x}}\sigma \sqrt{2}\frac{\Gamma(k/2)}{\Gamma((k-1)/2} \geq \norm{\perpart{x}}\sigma \sqrt{2} \left(\frac{k}{2} - \frac{3}{4} \right)^{1/2} \]
  \item[(b)] If ${1}_{\mathcal{A}}$ is an identity matrix with non-zeros diagonal entry $i$ iff $i \in \mathcal{A} \subset [k]$, and $|A| > 2$, then 
  \[ \mbold{E}\left[\norm{{1}_{\mathcal{A}} \perperp{W}\perpart{x}} \right] \geq \norm{\perpart{x}} \sigma \sqrt{2}\frac{\Gamma(|\mathcal{A}|/2)}{\Gamma((|\mathcal{A}|-1)/2)} \geq \norm{\perpart{x}}\sigma \sqrt{2} \left(\frac{|\mathcal{A}|}{2} - \frac{3}{4} \right)^{1/2} \]
\end{enumerate}
\end{lemma}

\begin{proof}
\begin{enumerate}
  \item[(a)] Let $U, V, \tilde{W}$ be as in Lemma \ref{lemma_matrix_decomp}. As $U, V$ are rotations, $\tilde{W}$ is also iid Gaussian. Furthermore for any fixed $W$, with $\tilde{a} = Va$, by taking inner products, and square-rooting, we see that $\norm{\tilde{W}\tilde{a}} = \norm{Wa}$. So in particular
  \[ \mbold{E}\left[\norm{\perperp{W}\perpart{x}} \right] = \mbold{E}\left[\norm{\perperp{\tilde{W}}\perpart{\tilde{x}}}
  \right] \]
  But from the definition of non-zero entries of $\perperp{\tilde{W}}$, and the form of $\perpart{\tilde{x}}$ (a zero entry in the first coordinate), it follows that $\perperp{\tilde{W}}\perpart{\tilde{x}}$ has exactly $k-1$ non zero entries, each a centered Gaussian with variance $(k -1)\sigma^2 \norm{\perpart{x}}^2$. By Lemma \ref{lemma_norm}, the expected norm is as in the statement. We then apply Theorem \ref{thm_gautschi} to get the lower bound.

  \item[(b)] First note we can view ${1}_{\mathcal{A}} \perperp{W} = \perperp{{1}_{\mathcal{A}} W}$. (Projecting down to a random (as $W$ is random) subspace of fixed size $|\mathcal{A}| = m$ and then making perpendicular commutes with making perpendicular and then projecting everything down to the subspace.) 
  
  So we can view $W$ as a random $m$ by $k$ matrix, and for $x, y$ as in Lemma \ref{lemma_matrix_decomp} (with $y$ projected down onto $m$ dimensions), we can again define $U, V$ as $k$ by $k$ and $m$ by $m$ rotation matrices respectively, and $\tilde{W} = UWV^T$, with analogous properties to Lemma 1. Now we can finish as in part (a), except that $\perperp{\tilde{W}}\tilde{x}$ may have only $m - 1$ entries, (depending on whether $y$ is annihilated by projecting down by${1}_{\mathcal{A}}$) each of variance $(k-1)\sigma^2 \norm{\perpart{x}}^2$. 
\end{enumerate}
\end{proof}

\begin{lemma}
\label{lemma_norm_gaussian}
\emph{Norm and Translation}
Let $X$ be a centered multivariate Gaussian, with diagonal covariance matrix, and $\mu$ a constant vector.
\[ \mbold{E}(\norm{X - \mu}) \geq \mbold{E}(\norm{X}) \]
\end{lemma}

\begin{proof}
The inequality can be seen intuitively geometrically: as $X$ has diagonal covariance matrix, the contours of the pdf of $\norm{X}$ are circular centered at $0$, decreasing radially. However, the contours of the pdf of $\norm{X - \mu}$ are shifted to be centered around $\norm{\mu}$, and so shifting back $\mu$ to $0$ reduces the norm.

A more formal proof can be seen as follows: let the pdf of $X$ be $f_X(\cdot)$. Then we wish to show
\[ \int_x \norm{x - \mu}f_X(x) dx \geq \int_x \norm{x}f_X(x) dx \]
Now we can pair points $x, -x$, using the fact that $f_X(x) = f_X(-x)$ and the triangle inequality on the integrand to get
\[ \int_{|x|}  \left(\norm{x - \mu} + \norm{-x -\mu} \right) f_X(x) dx \geq \int_{|x|}  \norm{2x} f_X(x) dx  = \int_{|x|} \left(\norm{x} + \norm{-x} \right) f_X(x) dx \]

\end{proof}

\subsection{Proof of Theorem}

We use $\node{d}{i}$ to denote the $i^{th}$ neuron in hidden layer $d$. We also let $x = z^{(0)}$ be an input, $h^{(d)}$ be the hidden representation at layer $d$, and $\phi$ the non-linearity. The weights and bias are called $W^{(d)}$ and $b^{(d)}$ respectively. So we have the relations
\begin{align}
    h^{(d)} = W^{(d)} z^{(d)} + b^{(d)}, 
    && z^{(d+1)} = \phi(h^{(d)}).
\end{align}

\begin{proof}
We first prove the zero bias case. To do so, it is sufficient to prove that
\[ \mbold{E}\left[\norm{\delta z^{(d+1)}(t)}\right] \geq O\left( \left( \frac{\sqrt{\sigma k}}{\sqrt{\sigma + k}} \right)^{d+1} \right) \norm{\delta z^{(0)}(t)} \tag{**} \]
as integrating over $t$ gives us the statement of the theorem. 

For ease of notation, we will suppress the $t$ in $z^{(d)}(t)$. 

We first write
\[ W^{(d)} = \perpart{W^{(d)}} + \parpart{W^{(d)}} \]
where the division is done with respect to $z^{(d)}$. Note that this means $h^{(d+1)} = \parpart{W^{(d)}} z^{(d)}$ as the other component annihilates (maps to $0$) $z^{(d)}$. 

We can also define $\mathcal{A}_{\parpart{W^{(d)}}} = \{ i : i \in [k], |\hidden{d+1}{i}| < 1 \}$ i.e. the set of indices for which the hidden representation is not saturated. 
Letting $W_i$ denote the $i$th row of matrix $W$, we now claim that:
\[ \mbold{E}_{W^{(d)}} \left[ \norm{\delta z^{(d+1)}} \right] = \mbold{E}_{\parpart{W}^{(d)}} \mbold{E}_{\perpart{W}^{(d)}} \left[ \left( \sum_{ i \in \mathcal{A}_{\parpart{W^{(d)}}}} ((\perpart{W^{(d)}})_i \delta z^{(d)} + (\parpart{W^{(d)}})_i \delta z^{(d)})^2 \right)^{1/2} \tag{*}  \right] \]

Indeed, by Lemma \ref{lemma_indep_proj} we first split the expectation over $W^{(d)}$ into a tower of expectations over the two independent parts of $W$ to get 
\[ \mbold{E}_{W^{(d)}} \left[ \norm{\delta z^{(d+1)}} \right] = \mbold{E}_{\parpart{W}^{(d)}} \mbold{E}_{\perpart{W}^{(d)}} \left[ \norm{ \phi(W^{(d)} \delta z^{(d)})} \right]    \]

But conditioning on $\parpart{W^{(d)}}$ in the inner expectation gives us $h^{(d+1)}$ and $\mathcal{A}_{\parpart{W^{(d)}}}$, allowing us to replace the norm over $\phi(W^{(d)} \delta z^{(d)})$ with the sum in the term on the right hand side of the claim.

Till now, we have mostly focused on partitioning the matrix $W^{(d)}$. But we can also set $\delta z^{(d)} = \parpart{\delta z^{(d)}} + \perpart{\delta z^{(d)}} $ where the perpendicular and parallel are with respect to $z^{(d)}$. In fact, to get the expression in (**), we derive a recurrence as below: 

\[  \mbold{E}_{W^{(d)}} \left[ \norm{ \perpart{\delta z^{(d+1)}}} \right] \geq O\left( \frac{\sqrt{\sigma k}}{\sqrt{\sigma + k}} \right)  \mbold{E}_{W^{(d)}} \left[  \norm{\perpart{\delta z^{(d)}} } \right] \]

To get this, we first need to define $\tilde{z}^{(d+1)} = {1}_{\mathcal{A}_{\parpart{W^{(d)}}}} h^{(d+1)} $ - the latent vector $h^{(d+1)}$ with all saturated units zeroed out.

We then split the column space of $W^{(d)} = {}^{\perp}W^{(d)} + {}^{\parallel}W^{(d)}$, where the split is with respect to $\tilde{z}^{(d+1)}$. Letting $\perpart{\delta z}^{(d+1)}$ be the part perpendicular to $z^{(d+1)}$, and $\mathcal{A}$ the set of units that are unsaturated, we have an important relation:

\textbf{Claim}
\[ \norm{\perpart{\delta z}^{(d+1)}}  \geq  \norm{{}^{\perp}W^{(d)} \delta z^{(d)} {1}_{\mathcal{A}}} \]
(where the indicator in the right hand side zeros out coordinates not in the active set.)

To see this, first note, by definition,
\[ \perpart{\delta z^{(d+1)}} = W^{(d)} \delta z^{(d)} \cdot {1}_{\mathcal{A}} - \langle W^{(d)} \delta z^{(d)} \cdot {1}_{\mathcal{A}}, \hat{z}^{(d+1)}  \rangle \hat{z}^{(d+1)} \tag{1}  \]
where the $\hat{\cdot}$ indicates a unit vector.

Similarly
\[ {}^{\perp}W^{(d)} \delta z^{(d)} = W^{(d)} \delta z^{(d)} - \langle W^{(d)} \delta z^{(d)}, \hat{\tilde{z}}^{(d+1)} \rangle \hat{\tilde{z}}^{(d+1)} \tag{2} \]

Now note that for any index $i \in \mathcal{A}$, the right hand sides of (1) and (2) are identical, and so the vectors on the left hand side agree for all $i \in \mathcal{A}$. In particular,
\[ \perpart{\delta z^{(d+1)}} \cdot {1}_{\mathcal{A}} = {}^{\perp}W^{(d)} \delta z^{(d)}  \cdot {1}_{\mathcal{A}} \]
Now the claim follows easily by noting that $\norm{\perpart{\delta z}^{(d+1)}} \geq \norm{  \perpart{\delta z^{(d+1)}} \cdot {1}_{\mathcal{A}}}$.

Returning to (*), we split $\delta z^{(d)} = \perpart{\delta z}^{(d)} + \parpart{\delta z}^{(d)}$, $\perpart{W}^{(d)} = \parperp{W}^{(d)} + \perperp{W}^{(d)}$ (and $\parpart{W}^{(d)}$ analogously), and after some cancellation, we have

\[ \mbold{E}_{W^{(d)}} \left[ \norm{\delta z^{(d+1)}} \right] = \mbold{E}_{\parpart{W}^{(d)}} \mbold{E}_{\perpart{W}^{(d)}} \left[ \left( \sum_{ i \in \mathcal{A}_{\parpart{W^{(d)}}}} \left( (\perperp{W}^{(d)} + \parperp{W}^{(d)})_i \perpart{\delta z}^{(d)} + (\perpar{W}^{(d)} + \parpar{W}^{(d)})_i \parpart{\delta z}^{(d)} \right)^2 \right)^{1/2} \right] \]

We would like a recurrence in terms of only perpendicular components however, so we first drop the $\parperp{W}^{(d)}, \parpar{W}^{(d)}$ (which can be done without decreasing the norm as they are perpendicular to the remaining terms) and using the above claim, have

\[ \mbold{E}_{W^{(d)}} \left[ \norm{ \perpart{\delta z^{(d+1)}}} \right] \geq \mbold{E}_{\parpart{W}^{(d)}} \mbold{E}_{\perpart{W}^{(d)}} \left[ \left( \sum_{ i \in \mathcal{A}_{\parpart{W^{(d)}}}} \left( (\perperp{W}^{(d)})_i \perpart{\delta z}^{(d)} + (\perpar{W}^{(d)})_i \parpart{\delta z}^{(d)} \right)^2 \right)^{1/2} \right] \]

But in the inner expectation, the term $\perpar{W}^{(d)} \parpart{\delta z}^{(d)}$ is just a constant, as we are conditioning on $\parpart{W}^{(d)}$. So using Lemma \ref{lemma_norm_gaussian} we have
\[
\mbold{E}_{\perpart{W}^{(d)}} \left[ \left( \sum_{ i \in \mathcal{A}_{\parpart{W^{(d)}}}} \left( (\perperp{W}^{(d)})_i \perpart{\delta z}^{(d)} + (\perpar{W}^{(d)})_i \parpart{\delta z}^{(d)} \right)^2 \right)^{1/2} \right] \geq 
 \mbold{E}_{\perpart{W}^{(d)}} \left[ \left( \sum_{ i \in \mathcal{A}_{\parpart{W^{(d)}}}} \left( (\perperp{W}^{(d)})_i \perpart{\delta z}^{(d)} \right)^2 \right)^{1/2} \right] \]

 We can then apply Lemma \ref{norm_projections} to get
 \[  \mbold{E}_{\perpart{W}^{(d)}} \left[ \left( \sum_{ i \in \mathcal{A}_{\parpart{W^{(d)}}}} \left( (\perperp{W}^{(d)})_i \perpart{\delta z}^{(d)} \right)^2 \right)^{1/2} \right] \geq \frac{\sigma}{\sqrt{k}} \sqrt{2}\frac{\sqrt{2|\mathcal{A}_{\parpart{W^{(d)}}}| - 3 }}{2} \mbold{E} \left[\norm{ \perpart{\delta z^{(d)}}}  \right] \]
 
 The outer expectation on the right hand side only affects the term in the expectation through the size of the active set of units. For ReLUs, $p = \mbold{P}(h^{(d+1)}_i > 0)$ and for hard tanh, we have $p = \mbold{P}(|h^{(d+1)}_i| < 1)$, and noting that we get a non-zero norm only if $|\mathcal{A}_{\parpart{W^{(d)}}}| \geq 2$ (else we cannot project down a dimension), and for $|\mathcal{A}_{\parpart{W^{(d)}}}| \geq 2$, 
 \[ \sqrt{2}\frac{\sqrt{2|\mathcal{A}_{\parpart{W^{(d)}}}| - 3 }}{2} \geq \frac{1}{\sqrt{2}} \sqrt{|\mathcal{A}_{\parpart{W^{(d)}}}|}    \]
 we get
 \begin{align*}
 \mbold{E}_{W^{(d)}} \left[ \norm{ \perpart{\delta z^{(d+1)}}} \right] & \geq \frac{1}{\sqrt{2}} \left( \sum_{j=2}^k \binom{k}{j} p^j (1 - p)^{k-j} \frac{\sigma}{\sqrt{k}} \sqrt{j}  \right) \mbold{E}\left[ \norm{ \perpart{\delta z^{(d)}}} \right]  
 \end{align*}
 We use the fact that we have the probability mass function for an $(k,p)$ binomial random variable to bound the $\sqrt{j}$ term:
 \begin{align*}
  \sum_{j=2}^k \binom{k}{j} p^j (1 - p)^{k-j} \frac{\sigma}{\sqrt{k}} \sqrt{j} &= - \binom{k}{1} p(1 -p)^{k-1} \frac{\sigma}{\sqrt{k}} + \sum_{j=0}^k \binom{k}{j} p^j (1 - p)^{k-j} \frac{\sigma}{\sqrt{k}} \sqrt{j} \\
  &= - \sigma \sqrt{k} p(1 -p)^{k-1} +  kp \cdot \frac{\sigma}{\sqrt{k}} \sum_{j=1}^k \frac{1}{\sqrt{j}} \binom{k-1}{j-1}p^{j-1}(1 - p)^{k-j}     
 \end{align*}

 But by using Jensen's inequality with $1/\sqrt{x}$, we get
 \[
 \sum_{j=1}^k \frac{1}{\sqrt{j}} \binom{k-1}{j-1}p^{j-1}(1 - p)^{k-j} \geq \frac{1}{\sqrt{\sum_{j=1}^k j \binom{k-1}{j-1}p^{j-1}(1 - p)^{k-j}}} = \frac{1}{\sqrt{(k - 1)p + 1}} \]
where the last equality follows by recognising the expectation of a binomial$(k-1, p)$ random variable. So putting together, we get
\[
 \mbold{E}_{W^{(d)}} \left[ \norm{ \perpart{\delta z^{(d+1)}}} \right] \geq \frac{1}{\sqrt{2}} \left( - \sigma \sqrt{k} p(1 -p)^{k-1} + \sigma \cdot \frac{\sqrt{k}p}{\sqrt{1 + (k-1)p}} \right) \mbold{E}\left[ \norm{ \perpart{\delta z^{(d)}}} \right]
  \tag{a} \]
  
From here, we must analyse the hard tanh and ReLU cases separately. First considering the hard tanh case:

To lower bound $p$, we first note that as $h^{(d+1)}_i$ is a normal random variable with variance $\leq \sigma^2$, if $A \sim \mathcal{N}(0, \sigma^2)$
\[ \mbold{P}(|h^{(d+1)}_i| < 1) \geq \mbold{P}(|A| < 1) \geq  \frac{1}{\sigma \sqrt{2 \pi}} \tag{b} \]
where the last inequality holds for $\sigma \geq 1$ and follows by Taylor expanding $e^{-x^2/2}$ around $0$. Similarly, we can also show that $p \leq \frac{1}{\sigma}$.

So this becomes
\begin{align*}
 \mbold{E}\left[ \norm{\delta z^{(d+1)}} \right] & \geq \left( \frac{1}{\sqrt{2}}  \left( \frac{1}{(2\pi)^{1/4}} \frac{\sqrt{\sigma k}}{\sqrt{\sigma \sqrt{2\pi} + (k - 1)}} - \sqrt{k}\left(1 - \frac{1}{\sigma} \right)^{k-1} \right) \right) \mbold{E}\left[ \norm{ \perpart{\delta z^{(d)}}} \right] \\
 & = O \left( \frac{\sqrt{\sigma k}}{\sqrt{\sigma + k}} \right) \mbold{E}\left[ \norm{ \perpart{\delta z^{(d)}}} \right]
\end{align*}
   
Finally, we can compose this, to get
\[ \mbold{E}\left[ \norm{\delta z^{(d+1)}} \right]
   \geq \left( \frac{1}{\sqrt{2}}  \left( \frac{1}{(2\pi)^{1/4}} \frac{\sqrt{\sigma k}}{\sqrt{\sigma \sqrt{2\pi} + (k - 1)}} - \sqrt{k}\left(1 - \frac{1}{\sigma} \right)^{k-1} \right) \right)^{d+1} c \cdot \norm{ \delta x(t)} \tag{c}\]
 
with the constant $c$ being the ratio of $\norm{ \perpart{\delta x(t)}}$ to $\norm{\delta x(t)}$. So if our trajectory direction is almost orthogonal to $x(t)$ (which will be the case for e.g. random circular arcs, $c$ can be seen to be $\approx 1$ by splitting into components as in Lemma \ref{lemma_matrix_decomp}, and using Lemmas \ref{lemma_norm}, \ref{norm_projections}.)

The ReLU case (with no bias) is even easier. Noting that for random weights, $p = 1/2$, and plugging in to equation (a), we get
\[
 \mbold{E}_{W^{(d)}} \left[ \norm{ \perpart{\delta z^{(d+1)}}} \right] \geq \frac{1}{\sqrt{2}} \left(  \frac{- \sigma \sqrt{k}}{ 2^k }+ \sigma \cdot \frac{\sqrt{k}}{\sqrt{2(k+1)}} \right) \mbold{E}\left[ \norm{ \perpart{\delta z^{(d)}}} \right]
  \tag{d} \]

But the expression on the right hand side has exactly the asymptotic form $O(\sigma \sqrt{k}/\sqrt{k+1})$, and we finish as in (c).

\paragraph{Result for non-zero bias}
In fact, we can easily extend the above result to the case of non-zero bias. The insight is to note that because $\delta z^{(d+1)}$ involves taking a \textit{difference} between $z^{(d + 1)}(t + d t)$ and $z^{(d + 1)}(t)$, the bias term does not enter at all into the expression for $\delta z^{(d+1)}$. So the computations above hold, and equation (a) becomes

\[ \mbold{E}_{W^{(d)}} \left[ \norm{ \perpart{\delta z^{(d+1)}}} \right] \geq \frac{1}{\sqrt{2}} \left( - \sigma_w \sqrt{k} p(1 -p)^{k-1} + \sigma_w \cdot \frac{\sqrt{k}p}{\sqrt{1 + (k-1)p}} \right) \mbold{E}\left[ \norm{ \perpart{\delta z^{(d)}}} \right]   \]

For ReLUs, we require $h_i^{(d+1)} = w_i^{(d+1)} z_i^{(d)} + b_i^{(d+1)} > 0$ where the bias and weight are drawn from  $\mathcal{N}(0, \sigma^2_b)$ and  $\mathcal{N}(0, \sigma^2_w)$ respectively. But with $p \geq 1/4$, this holds as the signs for $w, b$ are purely random. Substituting in and working through results in the \textit{same} asymptotic behavior as without bias. 

For hard tanh, not that as $h_i^{(d+1)}$ is a normal random variable with variance $\leq \sigma^2_w + \sigma^2_b$ (as equation (b) becomes
\[ \mbold{P}(|h^{(d+1)}_i| < 1) \geq  \frac{1}{\sqrt{(\sigma_w^2 + \sigma_b^2)} \sqrt{2 \pi}} \]

This gives Theorem \ref{thm_traj_growth_bias}
\[
 \mbold{E}\left[ \norm{\delta z^{(d+1)}} \right] \geq  O\left( \frac{\sigma_w}{(\sigma_w^2 + \sigma_b^2)^{1/4}} \cdot \frac{\sqrt{k}}{\sqrt{\sqrt{\sigma_w^2 + \sigma_b^2} + k}} \right) \mbold{E}\left[ \norm{ \perpart{\delta z^{(d)}}} \right]
\]

\end{proof}

\begin{figure}[h]
\centering
\adjincludegraphics[width=0.7\linewidth]{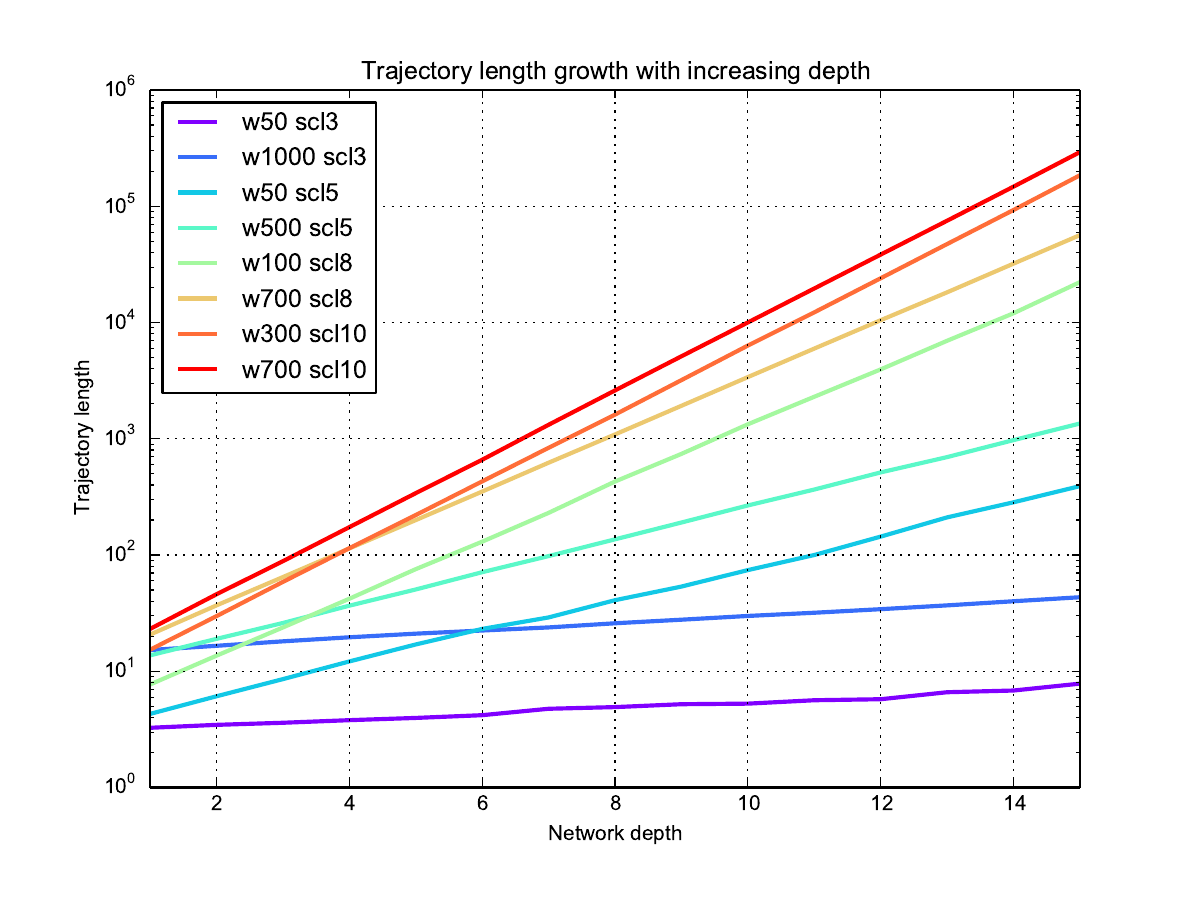}
\caption{The figure above shows trajectory growth with different initialization scales as a trajectory is propagated through  a fully connected network for MNIST, with Relu activations. Note that as described by the bound in Theorem \ref{thm_traj_growth_bias} we see that trajectory growth is 1) exponential in depth 2) increases with initialization scale and width, 3) increases faster with scale over width, as expected from $\sigma_w$ compared to $\sqrt{k/(k+1)}$ in the Theorem.
}
\label{MNIST_traj_length}
\end{figure}
\begin{figure}
\centering
\begin{tabular}{cc}
(a)\adjincludegraphics[width=0.4\linewidth]{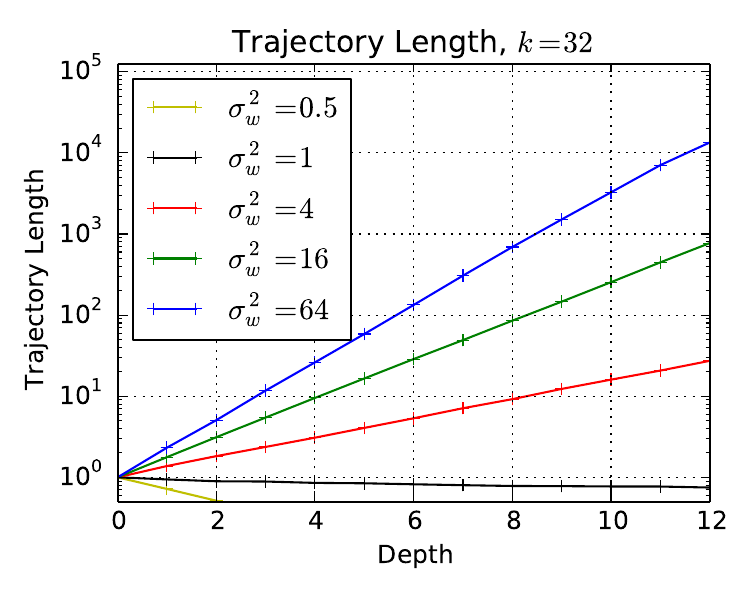} &
(b)\adjincludegraphics[width=0.4\linewidth]{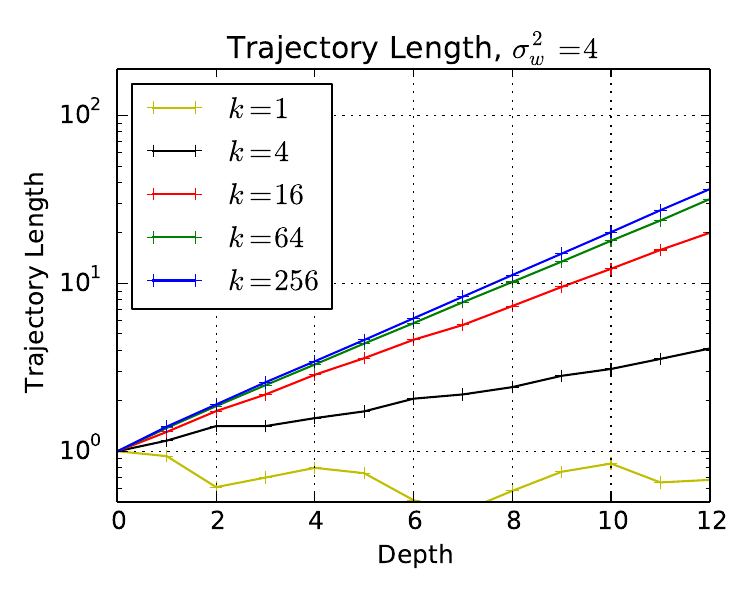} \\
(c)\adjincludegraphics[width=0.4\linewidth]{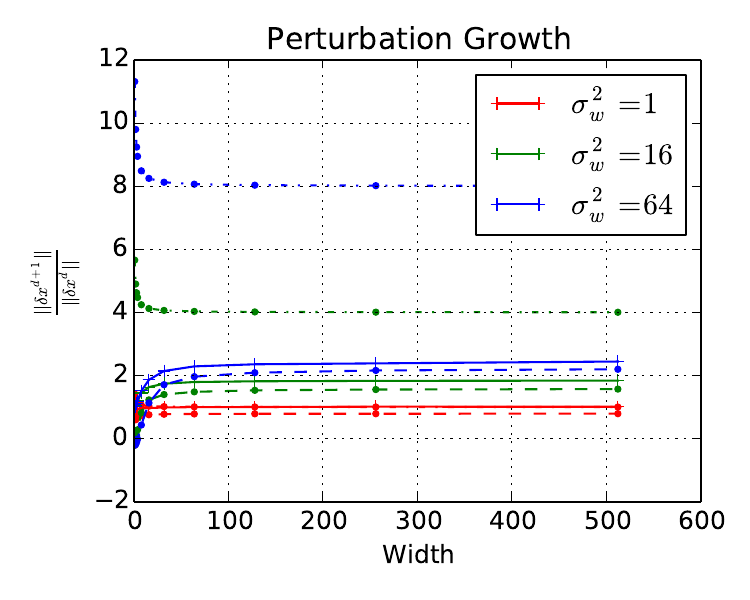} &
(d)\adjincludegraphics[width=0.4\linewidth]{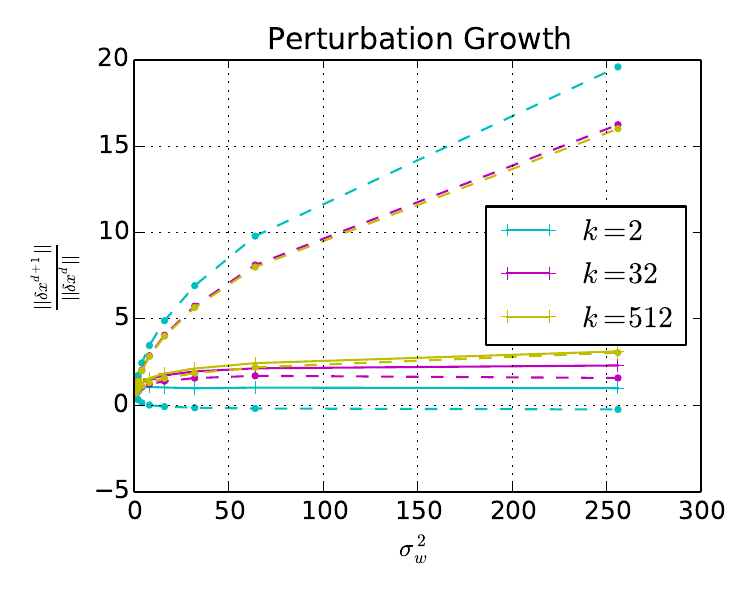} 
\end{tabular}
\caption{
The exponential growth of trajectory length with depth, in a random deep network with hard-tanh nonlinearities.
A circular trajectory is chosen between two random vectors.
The image of that trajectory is taken at each layer of the network, and its length measured.
\emph{(a,b)}
The trajectory length vs. layer, in terms of the network width $k$ and weight variance $\sigma_w^2$, both of which determine its growth rate.
\emph{(c,d)}
The average ratio of a trajectory's length in layer $d+1$ relative to its length in layer $d$.
The solid line shows simulated data, while the dashed lines show upper and lower bounds (Theorem \ref{thm_traj_growth_bias}).
Growth rate is a function of layer width $k$, and weight variance $\sigma^2_w$.
\label{fig growth}
}
\end{figure}

\paragraph{Statement and Proof of Upper Bound for Trajectory Growth for Hard Tanh}
Replace hard-tanh with a linear coordinate-wise identity map, $h^{(d+1)}_i = (W^{(d)}z^{(d)})_i + b_i$. This provides an upper bound on the norm. 
We also then recover a chi distribution with $k$ terms, each with standard deviation $\frac{\sigma_w}{k^{\frac{1}{2}}}$,
\begin{align}
 \mbold{E}\left[ \norm{ \delta z^{(d+1)}} \right] & \leq 
 \sqrt{2} \frac{
    \Gamma\left( (k+1) / 2\right)
 }{
    \Gamma\left( k / 2\right)
 }
 \frac{\sigma_w}{k^{\frac{1}{2}}}
 \norm{ \delta z^{(d)}}
 \\ & \leq
 \sigma_w \left( \frac{k+1}{k}
    \right)^\frac{1}{2} \norm{ \delta z^{(d)}}
,
\end{align} 
where the second step follows from \citep{laforgia2013some}, and holds for $k>1$.

\paragraph{Proof of Theorem \ref{thm_large_weight_limit}}

\begin{proof}
\textbf{For $\sigma_b = 0$:}

For hidden layer $d < n$, consider neuron $\node{d}{1}$. This has as input $\sum_{i=1}^k W^{(d-1)}_{i1} z^{(d-1)}_i$. As we are in the large $\sigma$ case, we assume that $|z^{(d-1)}_i| = 1$. Furthermore, as signs for $z^{(d-1)}_i$ and $W^{(d-1)}_{i1}$ are both completely random, we can also assume wlog that $z^{(d-1)}_i = 1$. For a particular input, we can define $\node{d}{1}$ as \textit{sensitive} to $\node{d-1}{i}$ if $\node{d-1}{i}$ transitioning (to wlog $-1$) will induce a transition in node $\node{d}{1}$. A sufficient condition for this to happen is if $|W_{i1}| \geq |\sum_{j \neq i} W_{j1}|$. But $X = W_{i1} \sim \mathcal{N}(0, \sigma^2/k)$ and $\sum_{j \neq i} W_{j1} = Y' \sim \mathcal{N}(0, (k-1)\sigma^2/k)$. So we want to compute $\mbold{P}(|X| > |Y'|)$. For ease of computation, we instead look at $\mbold{P}(|X| > |Y|)$, where $Y \sim \mathcal{N}(0, \sigma^2)$. 

But this is the same as computing $\mbold{P}(|X|/|Y| > 1) = \mbold{P}(X/Y < -1) + \mbold{P}(X/Y > 1)$. But the ratio of two centered independent normals with variances $\sigma_1^2, \sigma_2^2$ follows a Cauchy distribution, with parameter $\sigma_1/\sigma_2$, which in this case is $1/\sqrt{k}$. Substituting this in to the cdf of the Cauchy distribution, we get that
\[ \mbold{P}\left(\frac{|X|}{|Y|} > 1 \right) = 1 - \frac{2}{\pi} \arctan(\sqrt{k}) \]
Finally, using the identity $\arctan(x) + \arctan(1/x)$ and the Laurent series for $\arctan(1/x)$, we can evaluate the right hand side to be $O(1/\sqrt{k})$. In particular
\[ \mbold{P}\left(\frac{|X|}{|Y|} > 1 \right) \geq O\left(\frac{1}{\sqrt{k}} \right)  \tag{c}\]
This means that in expectation, any neuron in layer $d$ will be sensitive to the transitions of $\sqrt{k}$ neurons in the layer below. Using this, and the fact the while $\node{d-1}{i}$ might flip very quickly from say $-1$ to $1$, the gradation in the transition ensures that neurons in layer $d$ sensitive to $\node{d-1}{i}$ will transition at distinct times, we get the desired growth rate in expectation as follows: 

Let $T^{(d)}$ be a random variable denoting the number of transitions in layer $d$. And let $T_i^{(d)}$ be a random variable denoting the number of transitions of neuron $i$ in layer $d$. Note that by linearity of expectation and symmetry, $\mbold{E} \left[ T^{(d)} \right] = \sum_i \mbold{E} \left[ T^{(d)}_i \right] = k \mbold{E} \left[ T^{(d)}_1 \right]$ 

Now, $ \mbold{E} \left[T^{(d+1)}_1 \right] \geq \mbold{E} \left[ \sum_i 1_{(1,i)} T^{(d)}_i \right] =  k \mbold{E} \left[ 1_{(1,1)}T^{(d)}_1 \right]$ where $1_{(1,i)}$ is the indicator function of neuron $1$ in layer $d+1$ being sensitive to neuron $i$ in layer $d$.

But by the independence of these two events, $\mbold{E} \left[ 1_{(1,1)}T^{(d)}_1 \right] = \mbold{E} \left[1_{(1,1)} \right] \cdot \mbold{E} \left[T^{(d)}_1 \right]$. But the firt time on the right hand side is $O(1/\sqrt{k})$ by (c), so putting it all together, $ \mbold{E} \left[T^{(d+1)}_1 \right] \geq \sqrt{k} \mbold{E} \left[T^{(d)}_1 \right] $.

Written in terms of the entire layer, we have $\mbold{E} \left[T^{(d+1)} \right] \geq \sqrt{k} \mbold{E} \left[T^{(d)} \right] $ as desired.

\textbf{For $\sigma_b > 0$:}

We replace $\sqrt{k}$ with $\sqrt{k( 1 + \sigma_b^2/\sigma_w^2)}$, by noting that $Y \sim \mathcal{N}(0, \sigma_w^2 + \sigma_b^2)$. This results in a growth rate of form $O(\sqrt{k}/\sqrt{1 + \frac{\sigma_b^2}{\sigma_w^2}})$.
\end{proof}

\subsection{Dichotomies: a natural dual} \label{sec_function_space}

Our measures of expressivity have mostly concentrated on sweeping the input along a trajectory $x(t)$ and taking measures of $F_A(x(t); W)$. Instead, we can also sweep the weights $W$ along a trajectory $W(t)$, and look at the consequences (e.g. binary labels -- i.e. \textit{dichotomies}), say for a fixed set of inputs $x_1,...,x_s$. 

In fact, after random initialization, sweeping the first layer weights is statistically very similar to sweeping the input along a trajectory $x(t)$. In particular, letting $W'$ denote the first layer weights, for a particular input $x_0$, $x_0 W'$ is a vector, each coordinate is iid, $\sim \mathcal{N}(0, ||x_0||^2 \sigma_w^2)$. Extending this observation, we see that (providing norms are chosen appropriately), $x_0 W' \cos(t) + x_1 W' \sin(t)$ (fixed $x_0, x_1, W$) has the same distribution as $x_0W'_0 \cos(t) + x_0 W'_1 \sin(t)$ (fixed $x_0, W'_0, W'_1$). 

\begin{figure}
\centering
\begin{tabular}{cc}
(a)\adjincludegraphics[width=0.5\linewidth]{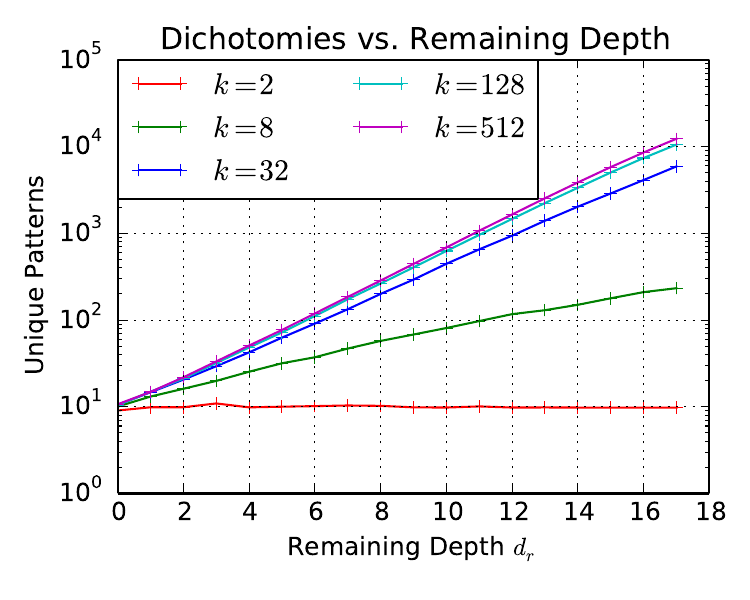} &
(b)\adjincludegraphics[width=0.5\linewidth]{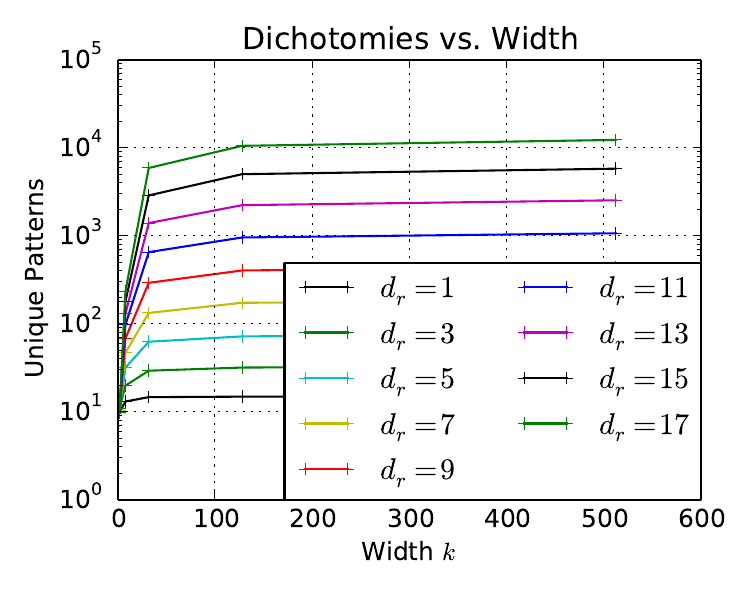}
\end{tabular}
\caption{We sweep the weights $W$ of a layer through a trajectory $W(t)$ and count the number of labellings over a set of datapoints. When $W$ is the first layer, this is statistically identical to sweeping the input through $x(t)$ (see Appendix). Thus, similar results are observed, with exponential increase with the depth of an architecture, and much slower increase with width.
Here we plot the number of classification dichotomies over $s=15$ input vectors achieved by sweeping the first layer weights in a hard-tanh network along a one-dimensional great circle trajectory. We show this  
\emph{(a)} as a function of  depth for several widths, and
\emph{(b)} as a function of width for several  depths.
All networks were generated with weight variance $\sigma_w^2 = 8$, and bias variance $\sigma_b^2 = 0$.
\label{fig patterns depth and width}
}
\end{figure}

So we expect that there will be similarities between results for sweeping weights and for sweeping input trajectories, which we explore through some synthetic experiments, primarily for hard tanh, in Figures \ref{fig power by layer}, \ref{fig transitions to patterns}. We find that the proportionality of transitions to trajectory length extends to dichotomies, as do results on the expressive power afforded by remaining depth.

For non-random inputs and non-random functions, this is a well known question upper bounded by the Sauer-Shelah lemma \citep{sauer1972density}. We discuss this further in Appendix \ref{sec VC}. In the random setting, the statistical duality of weight sweeping and input sweeping suggests a direct proportion to transitions and trajectory length for a fixed input. Furthermore, if the $x_i \in S$ are sufficiently uncorrelated (e.g. random) class label transitions should occur independently for each $x_i$ Indeed, we show this in Figure \ref{fig patterns depth and width}.

\section{Addtional Experiments from Section \ref{experiments}}
Here we include additional experiments from Section \ref{experiments}

\begin{figure}
\centering
\adjincludegraphics[width=0.4\linewidth]{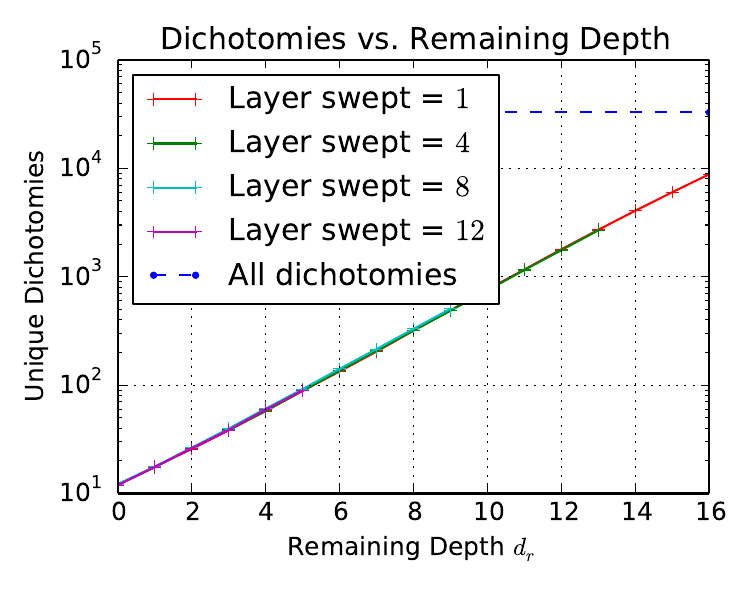}
\caption{
Expressive power depends only on remaining network depth. 
Here we plot the number of dichotomies achieved by sweeping the weights in different network layers through a 1-dimensional great circle trajectory, as a function of the remaining network depth. 
The number of achievable dichotomies does not depend on the total network depth, only on the number of layers above the layer swept. 
All networks had width $k=128$, weight variance $\sigma_w^2 = 8$, number of datapoints $s=15$, and hard-tanh nonlinearities. 
The blue dashed line indicates all $2^s$ possible dichotomies for this random dataset.
\label{fig power by layer}
}
\end{figure}

\begin{figure}
\centering
\adjincludegraphics[width=1.07\linewidth]{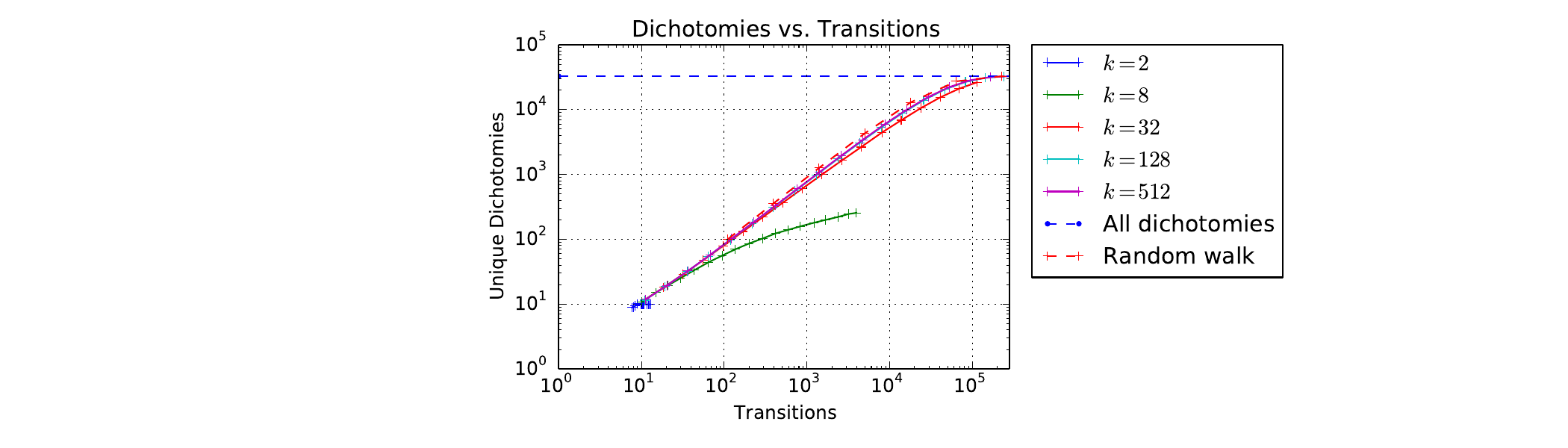}
\caption{
Here we plot the number of unique dichotomies that have been observed as a function of the number of transitions the network has undergone. 
Each datapoint corresponds to the number of transitions and dichotomies for a hard-tanh network of a different depth, with the weights in the first layer undergoing interpolation along a great circle trajectory $W^{(0)}(t)$. 
We compare these plots to a random walk simulation, where at each transition a single class label is flipped uniformly at random.
Dichotomies are measured over a dataset consisting of $s=15$ random samples, and all networks had weight variance $\sigma_w^2 = 16$.
The blue dashed line indicates all $2^s$ possible dichotomies.
\label{fig transitions to patterns}
}
\end{figure}

\begin{figure}
\centering
\adjincludegraphics[width=0.75\linewidth]{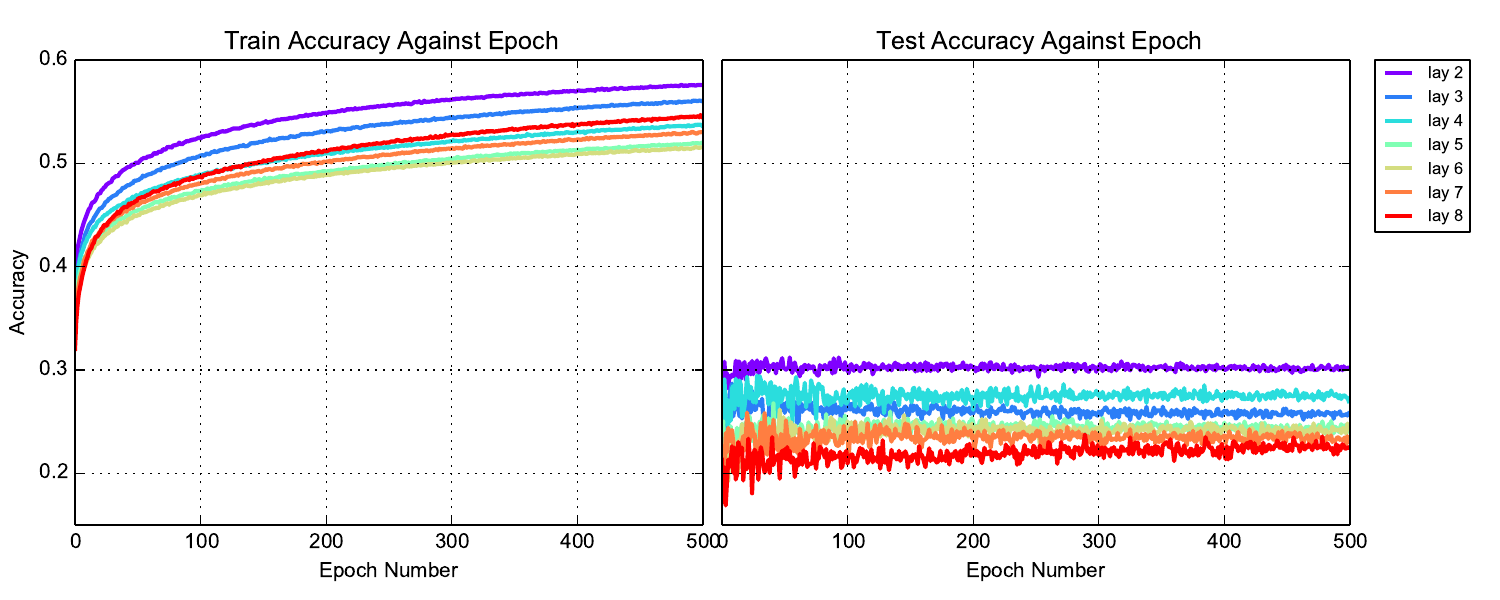}
\caption{
We repeat a similar experiment in Figure \ref{mnist_ffn} with a fully connected network on CIFAR-10, and mostly observe that training lower layers again leads to better performance, although, as expected, overall performance is impacted by training only a single layer. The networks had width $k=200$, weight variance $\sigma_w^2 = 1$, and hard-tanh nonlinearities. We again only train from the second hidden layer on so that the number of parameters remains fixed. The theory only applies to training error (the ability to fit a function), and generalisation accuracy remains low in this very constrained setting.
\label{cifar_ffn}
}
\end{figure}

\begin{figure}
\centering
\adjincludegraphics[width=0.8\linewidth]{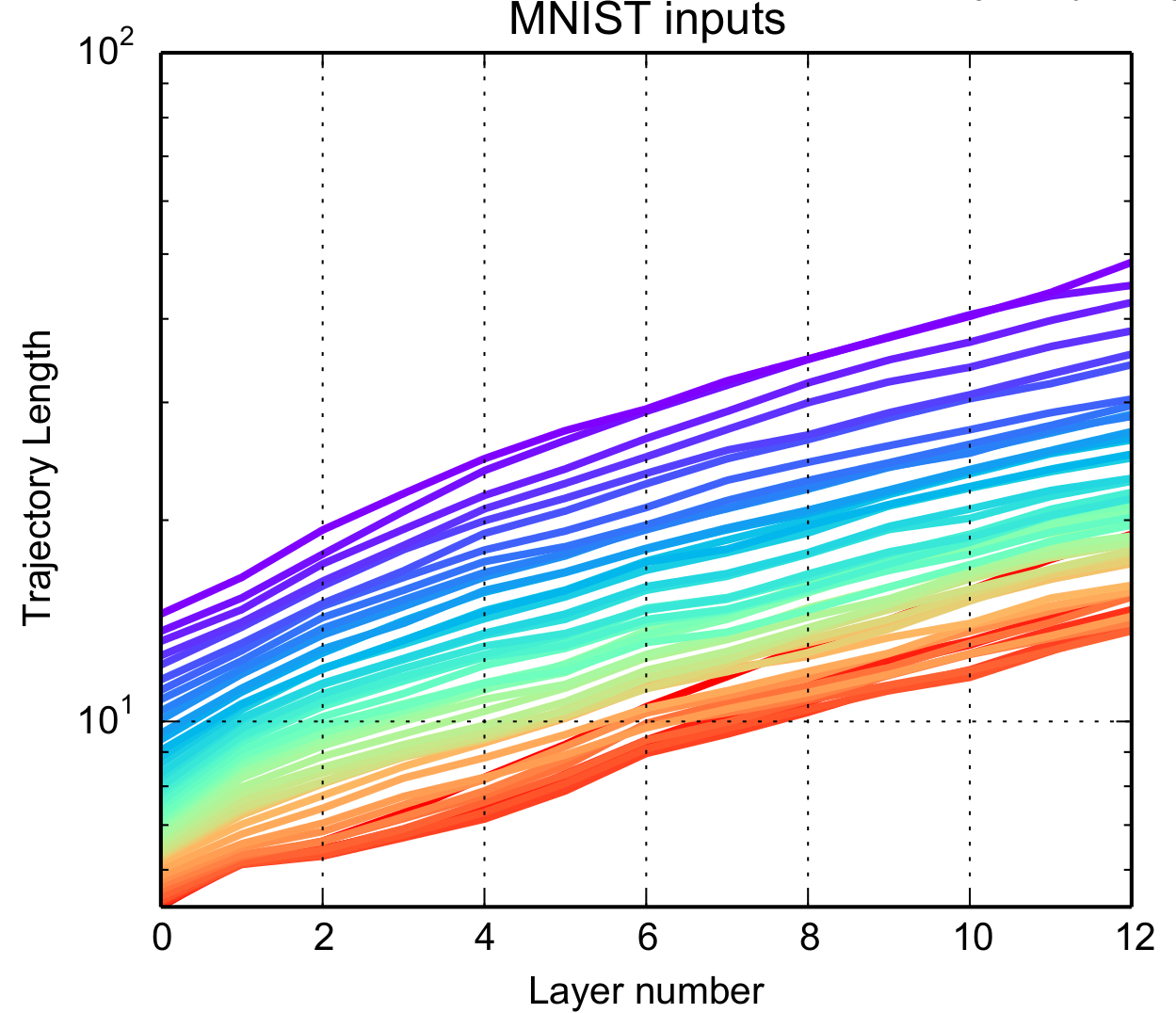}
\caption{
Training increases the trajectory length for smaller initialization values of $\sigma_w$. 
This experiment plots the growth of trajectory length as a circular interpolation between two MNIST datapoints is propagated through the network, at different train steps. Red indicates the start of training, with purple the end of training. We see that the training process \textit{increases} trajectory length, likely to increase the expressivity of the input-output map to enable greater accuracy. 
\label{mnist_traj_trans_sigma_3}
}
\end{figure}

\end{document}